\documentclass{article}
\usepackage{ijcai24}

\usepackage{times}
\usepackage{soul}
\usepackage{url}
\usepackage[hidelinks]{hyperref}
\usepackage[utf8]{inputenc}
\usepackage[small]{caption}
\usepackage{graphicx}
\usepackage{amsmath}
\usepackage{amsthm}
\usepackage{booktabs}
\usepackage{algorithm}
\usepackage{algorithmic}
\usepackage[switch]{lineno}
\usepackage{amsfonts}
\usepackage{mathrsfs}
\usepackage{tikz}
\usetikzlibrary{automata, positioning}
\usepackage{xcolor}
\usepackage{mathtools}

\usepackage{amssymb}
\usepackage[inline]{enumitem}
\usepackage{wrapfig}
\usepackage{adjustbox}

\newcommand{\algname}{LARM-RM}

\newcommand{\init}{I}

\newcommand{\rmLabels}{\mathcal P}
\newcommand{\rmLabelingFunction}{L}

\newcommand{\machine}{\mathcal{A}}
\newcommand{\mealyStates}{V}

\newcommand{\mealyStateStyle}[1]{\fsm{#1}}
\newcommand{\mealyCommonState}{\fsm{q}}
\newcommand{\mealyCommonInput}{l}        
\newcommand{\mealyInputAlphabet}{\ensuremath{{2^\mathcal{P}}}}
\newcommand{\mealyInit}{{\mealyCommonState_\init}}

\newcommand{\mealyOutputAlphabet}{\ensuremath{M}}
\newcommand{\mealyOutput}{\sigma}
\newcommand{\mealyTransition}{\delta}

\newcommand{\setOfSeenRewards}{R}

\newcommand{\mdp}{\mathcal{M}}
\newcommand{\mdpStates}{S}
\newcommand{\mdpCommonState}{s}
\newcommand{\mdpInit}{\mdpCommonState_\init}
\newcommand{\mdpActions}{A}
\newcommand{\mdpCommonAction}{a}
\newcommand{\mdpProb}{p}
\newcommand{\mdpRewardFunction}{R}
\newcommand{\mdpDiscount}{\gamma}

\newcommand{\mdpLabel}{\ell}
\newcommand{\mdpRewards}{r}
\newcommand{\trajectory}[1]{\ensuremath{\mdpCommonState_0 \mdpCommonAction_1 \mdpCommonState_1\ldots \mdpCommonAction_#1 \mdpCommonState_{#1}}}
\newcommand{\labelSequence}[1]{\ensuremath{\mdpLabel_1 \mdpLabel_2\ldots \mdpLabel_{#1}}}
\newcommand{\rewardSequence}[1]{\ensuremath{\mdpRewards_1 \mdpRewards_2\ldots \mdpRewards_{#1}}}

\newcommand{\automaton}[1]{\mathfrak{#1}}
\newcommand{\Automaton}{\automaton{A}}
\newcommand{\Advice}{\ensuremath{D}}

\newcommand{\eplen}{\mathit{episode_{length}}}
\newcommand{\fsm}[1]{\mathsf{#1}}
\newcommand{\run}[1]{\xrightarrow{#1}}
\newcommand{\Pref}{\mathit{Pref}}

{
	\begin{tikzpicture}[shorten >=1pt,node distance=2cm,on grid,auto, thick]
}%
{
	\end{tikzpicture}
} 

\newcommand{\Language}{\ensuremath{\mathcal{L}}}

\newtheorem{theorem}{Theorem}
\newtheorem{definition}{Definition}
\newtheorem{lemma}{Lemma}
\newtheorem{corollary}{Corollary}

\title{Using Large Language Models to Automate and Expedite Reinforcement Learning with Reward Machine }

\author{
Shayan Meshkat Alsadat$^1$
\and
 Jean-Rapha\"el Gaglione$^2$\and
 Daniel Neider $^3$ \and
Ufuk Topcu$^{2}$\And
Zhe Xu$^1$\\
\affiliations
$^1$Arizona State University\\
$^2$University of Texas at Austin\\
$^3$Technical University of Dortmund\\
\emails
\{smeshka1, xzhe1\}@asu.edu,
\{jr.gaglione, utopcu\}@utexas.edu,
daniel.neider@cs.tu-dortmund.de
}

\begin{document}

\maketitle

\begin{abstract}
    We present LARL-RM (\textbf{L}arge language model-generated \textbf{A}utomaton for \textbf{R}einforcement \textbf{L}earning with \textbf{R}eward \textbf{M}achine) algorithm in order to encode high-level knowledge into reinforcement learning using automaton to expedite the reinforcement learning. Our method uses Large Language Models (LLM) to obtain high-level domain-specific knowledge using prompt engineering instead of providing the reinforcement learning algorithm directly with the high-level knowledge which requires an expert to encode the automaton. We use chain-of-thought and few-shot methods for prompt engineering and demonstrate that our method works using these approaches. Additionally, LARL-RM allows for fully closed-loop reinforcement learning without the need for an expert to guide and supervise the learning since LARL-RM can use the LLM directly to generate the required high-level knowledge for the task at hand. We also show the theoretical guarantee of our algorithm to converge to an optimal policy. We demonstrate that LARL-RM speeds up the convergence by $30\%$ by implementing our method in two case studies.
\end{abstract}

\section{Introduction}

Large Language Models (LLM) encode various types of information including domain-specific knowledge which enables them to be used as a source for extracting specific information about a domain. This knowledge can be used for planning and control of various systems such as autonomous systems by synthesizing the control policies ~\cite{seff2023motionlm}. These policies are essentially dictating the action of the agent in the environment in a high-level format.

In general, the extraction of the domain-specific knowledge may not be a straightforward task since prompting the LLM without a proper prompt technique may result in hallucinations \cite{jang2023can}, even with proper techniques it is also known LLMs can hallucinate. There are several methods suggested for prompt engineering ranging from zero-shot \cite{kojima2022large} all the way to training the pre-trained language model on a domain-specific dataset (fine-tuning) \cite{rafailov2023direct} to obtain more accurate responses. The latter approach is more involved and resource intensive since it requires a dataset related to the domain as well as enough computation power to train the pre-trained language model on the dataset \cite{ding2023parameter}. Moreover, not all of the language models are available for fine-tuning due to their providers' policy. Hence, fine-tuning is a less desirable approach as long as the other prompt engineering methods can produce accurate responses.

Our algorithm LARL-RM (large language model-based automaton for reinforcement learning) uses prompt engineering methods to extract the relevant knowledge from the LLM. This knowledge is then incorporated into the reinforcement learning (RL) algorithm in the form of high-level knowledge to expedite the RL and guide it to reach the optimal policy faster. We integrate the LLM-encoded high-level knowledge into the RL using deterministic finite automata (DFA). We have the following contributions to the field: (a) LARL-RM allows for integration of the high-level knowledge from a specific domain into the RL using LLM outputs directly. (b) Our proposed method is capable of adjusting the instructions of the LLM by closing the loop in learning and updating the prompt based on the chain-of-thought method to generate the automata when counterexamples are encountered during the policy update. (c) LARL-RM is capable of learning the reward machine and expediting the RL by providing LLM-generated DFA to the RL algorithm. We also show that LARL-RM speeds up the RL by $30\%$ by implementing it in two case studies.

\noindent \textbf{Related work.} Generative language models (GLM) such as the GPT models have been shown to encode world knowledge and be capable of generating instructions for a task \cite{hendrycks2020measuring}. Researchers in \cite{west2021symbolic} have demonstrated the advances in GLM allow to construct task-relevant knowledge graphs. \textbf{Generating automaton from GLM.} Generating automaton from a generative language model for a specific task is shown in \cite{yang2022automaton}; however, the method proposed by the authors called GLM2FSA uses verb and verb phrases to and maps them onto a predefined set of verbs in order to convert the instructions given by the GLM to an automaton which in turn reduces the generality of the method to be applied to any other domain and still requires an expert to pre-define this set. GLMs are also used to extract task-relevant semantic knowledge and generate plans to complete a task in robotics \cite{vemprala2023chatgpt}. \textbf{Learning reward machine.} Researchers in \cite{neider2021advice} have shown a method called AdvisoRL for RL where the algorithm is capable of learning the reward machine; however, an expert is required to guide the learning process through provision of so-called advice. In our method this is replaced by using LLM.

\section{Preliminaries}

In this section, we define our notations and introduce the necessary background for reinforcement learning, reward machines, and finite deterministic automata (DFA).

\begin{definition}[labeled Markov decision process]
    A labeled Markov decision process is a tuple $\mathcal{M} = \langle S, s_{I}, A, p, R, \gamma, \mathcal{P}, L \rangle$ where $S$ is a finite set of states, $s_{I} \in S$ is the agent's initial state, $A$ is a finite set of actions, $p: S \times A \times S \mapsto \left[0, 1 \right]$, reward function $R: S \times A \times S \mapsto \mathbb{R}$, discount factor $\gamma \in [0, 1]$, a finite set of propositions $\mathcal{P}$, and a labeling function $L: S \times A \times S \mapsto 2^{\mathcal{P}}$ that determines the high-level events that agent encounters in the environment. 
\end{definition}

A policy is a function mapping states $S$ to actions in $A$ with a probability distribution, meaning the agent at state $s \in S$ will choose action $a \in A$ with probability $\pi(s, a)$ using the policy $\pi$ that leads to a new state $s^{\prime}$. The probability of the state $s^{\prime}$ is defined by $p(s, a, s^{\prime})$. \textit{Trajectory} $s_{0} a_{0} s_{1} \ldots s_{k} a_{k} s_{k+1}$ where $k \in \mathbb{N}$ is the resulting sequence of states and actions during this stochastic process. Its corresponding  \textit{label} sequence is $l_{0} l_{1} \ldots l_{k}$ where $L(s_{i}, a_{i}, s_{i+1}) = l_{i} \ \forall \ i \le k$. \textit{Trace} is the pair of labels and rewards $(\lambda, \rho) \coloneqq (l_{1} l_{2} \ldots l_{k}, r_{1} r_{2} \ldots r_{k})$ where the payoff of the agent is $\sum_{i} \gamma^{i} r_{i}$.

We exploit deterministic finite automaton (DFA) to encode the high-level knowledge obtained from the LLM to reinforcement learning in order to speed up the process of obtaining the optimal policy $\pi^{*}$. A DFA consists of an input and output alphabet which we can use to steer the learning process.

\begin{definition}[Deterministic finite automata]
    A DFA is a tuple $\mathcal{D} = \langle H, h_{I}, \Sigma, \delta, F \rangle $ where
    $H$ is the finite set of states,
    $h_{o}$ is the initial state, and
    $\Sigma$ is the input alphabet (here we consider $\Sigma = 2^{\mathcal{P}}$),
    $\delta: H \times \Sigma \mapsto H$ is the transition function between the states, and
    $F \subseteq H$ is the set of accepting states.
\end{definition}

We use deterministic finite state automata to model the LLM-generated DFA. The main idea of the LLM-generated DFA is to assist the algorithm with the learning of the reward machines through the provision of information about which label sequence may result in a positive reward and which one will definitely result in a non-positive reward  \cite{neider2021advice}.

Reward machines are a type of Mealy machine that is used to encode a non-Markovian reward function. They are automatons that receive a label from the environment and respond with a reward as well as transition to their next state. Their input alphabet is a set of propositional variables (from the set of $\mathcal{P}$) and their output alphabet is a set of real numbers.

\begin{definition}[Reward machine]
    A reward machine is shown by $\mathcal{A} = \left(V, \nu_{I}, 2^{\mathcal{P}}, M, \delta, \sigma \right)$ where $V$ is a finite set of states, $\nu_{I} \in V$ is the initial state, input alphabet $2^{\mathcal{P}}$, output alphabet $M \subseteq \mathbb{R}$, transition function $\delta: V \times 2^{\mathcal{P}} \mapsto V$, and output function $\sigma: V \times 2^{\mathcal{P}} \mapsto M$.
\end{definition}

Applying the reward machine $\mathcal{A}$ on a sequence of labels $l_{1} l_{2} \ldots l_{k} \in (2^{\mathcal{P}})^*$ will create a sequence of states $\nu_{0}(l_{1}, r_{1}) \nu_{1}(l_{2}, r_{2}) \ldots \nu_{k-1}(l_{k}, r_{k}) \nu_{k}$ where $\nu_{0} = \nu_{I}$ for all $i \in \{0, \ldots, k-1 \}$. Transition to next state of the reward machine $\delta(\nu_{i}, l_{i+1}) = \nu_{i+1}$ results in a reward of $\sigma(\nu_{i}, l_{i+1}) = r_{i+1}$. Hence, a label sequence of $l_{1} l_{2} \ldots l_{k}$ corresponds to a sequence of rewards produced by reward machine $\mathcal{A}(l_{1} l_{2} \ldots l_{k}) = r_{1} r_{2} \ldots r_{k}$. Therefore, reward function $R$ of an MDP process for every trajectory $s_{0} a_{0} s_{1} \ldots s_{k} a_{k} s_{k+1}$ corresponds to a label sequence of $l_{1} l_{2} \ldots l_{k}$ that encodes the reward machine $\mathcal{A}$, agent then receives reward sequence of $\mathcal{A}(l_{1} l_{2} \ldots l_{k})$.

\begin{figure}[H]
    \centering
    \includegraphics[scale=0.26]{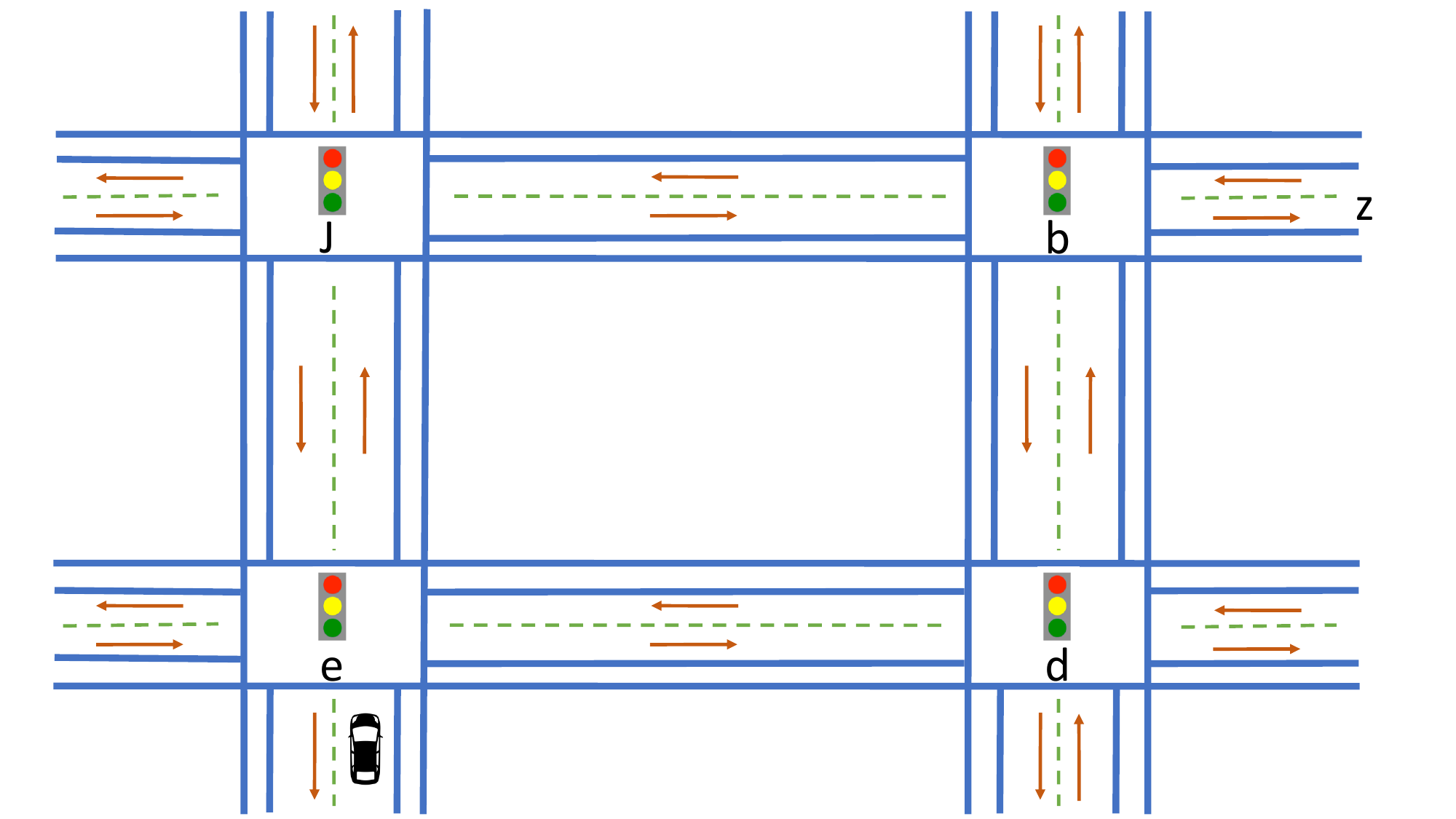}
    \caption{An autonomous car (agent) must first go to intersection $J$ and then make a right turn to go to intersection $b$. Agent must check for the green traffic light $g$, car $c$, and pedestrian $p$. }
    \label{fig:environment}
\end{figure}

\noindent \textbf{Motivating Example.} We consider an autonomous car on American roads with the set of actions $A=\{\textit{up}, \textit{down}, \textit{right}, \textit{left}, \textit{stay} \}$ (see Figure \ref{fig:environment}) where the agent must reach intersection $J$ then making a right turn when the traffic light $g$ is green and if there is no car $c$ on the left and no pedestrian $p$ on the right. Then the agent reaches the intersection$b$ where it needs to go straight when the traffic light is green if there are no cars and pedestrians. The corresponding DFA for this motivating example is shown in Figure \ref{fig:DFA_automata_example}.

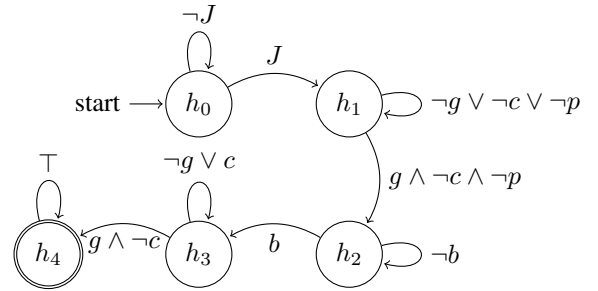
\begin{figure}[H]
  \centering
  \begin{tikzpicture}[shorten >=1pt, node distance=2cm, on grid, auto]
    \node[state, initial] (h0) {$h_0$};
    \node[state] (h1) [right=of h0] {$h_1$};
    \node[state] (h2) [below=of h1] {$h_2$};
    \node[state] (h3) [left=of h2] {$h_3$};
    \node[state, accepting] (h4) [left=of h3] {$h_4$};

    \path[->]
    (h0) edge [bend left] node {$J$} (h1)
    (h0) edge [loop above] node {$\neg J$} (h0) 
    (h1) edge [bend left] node {$g \wedge \neg c \wedge \neg p$} (h2)
    (h1) edge [loop right] node {$\neg g \vee \neg c \vee \neg p$} (h1)
    (h2) edge [loop right] node {$\neg b$} (h2)
    (h2) edge [bend right] node {$b$} (h3)
    (h3) edge [loop above] node {$\neg g \vee c$} (h3)
    (h3) edge [bend right] node {$g \wedge \neg c$} (h4)
    (h4) edge [loop above] node {$\top$} (h4);
  \end{tikzpicture}
  
  \caption{DFA of the motivating example where the agent must reach intersection $J$ while avoiding cars and pedestrians. }
  \label{fig:DFA_automata_example}
\end{figure}

\section{Generating Domain Specific Knowledge Using LLM}
Large Language Models (LLM) encode rich world and domain-specific knowledge \cite{yang2023fine} which allows for the extraction of the relevant information through appropriate prompting and has shown that can even surpass fine-tuning LLMs in some cases \cite{wei2022chain}. Hence, we use prompt engineering to extract the relevant domain knowledge information from the LLM and provide it in the form of LLM-generated DFA \cite{neider2021advice} to expedite the reinforcement learning. An LLM-generated DFA can be thought of as a means to provide the process of learning reward machine with information about which label sequence could be promising in such a way that it results in a reward.

Applying a DFA $\mathcal{D}$ on a sequence $\lambda = l_{1} l_{2} \ldots l_{n} \in \Sigma^\star$ will result in a sequence of states $h_{0} h_{1} \ldots h_{n}$ where $h_{0} = h_{I}$ and $h_{i+1} = \delta(h_{i}, l_{i}) \ \forall \ i \in \{1, \ldots, n \}$.
We define $\Language(\mathcal{D}) = \left\{ l_{1} l_{2} \ldots l_{n} \in \Sigma^\star \middle| h_{n} \in F \right\}$, also commonly referred to as the formal language accepted by $\mathcal{D}$.
A LLM-generated DFA could result in a \textit{positive reward}, meaning that DFA $\mathcal{D}$ can be considered as a LLM-generated DFA if the label sequence $\lambda \in \Language(\mathcal{D})$ could result in a positive reward; while, the label sequences $\lambda \notin \Language(\mathcal{D})$ \textit{should not} receive a positive reward; therefore, it is a binary classifier that points out the promising explorations \cite{neider2021advice}. Our proposed algorithm uses the LLM-generated DFAs to narrow down the search space for learning the ground truth reward machine. Hence, expediting the reinforcement learning; however, in our method, the LLM-generated DFA is generated automatically for a specific domain using the LLM. It is known that the LLMs hallucinate; therefore, our proposed algorithm is capable of ignoring the LLM-generated DFA $\mathcal{D}$ if it is \textit{incompatible} the underlying ground truth reward machine.

\begin{definition}[LLM-generated DFA compatibility]
    An LLM-generated DFA $\mathcal{D}$ is compatible with the ground truth reward machine $\mathcal{A}$ if for all label sequences $l_{1} l_{2} \ldots l_{k} \in (2^{\mathcal{P}})^*$ with reward sequence $\mathcal{A}(l_{1} l_{2} \ldots l_{k}) = r_{1} r_{2} \ldots r_{k}$ \text{it holds that} $r_{k} > 0$ \text{implies} $l_{1} l_{1} \ldots l_{k} \in \Language(\mathcal{D})$.
\end{definition}

We demonstrate an LLM-generated DFA for the motivating example in Figure \ref{fig:advice_dfa_example}.

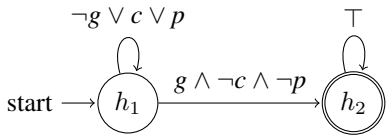
\begin{figure}[H]
  \centering
  \begin{tikzpicture}[shorten >=1pt, node distance=3cm, on grid, auto, state/.style={circle, draw, minimum size=2em}]
    \node[state, initial] (h1) {$h_1$};
    \node[state, accepting] (h2) [right=of h1] {$h_2$};

    \path[->]
    (h1) edge [loop above] node {$\neg \textit{g} \vee \textit{c} \vee \textit{p}$} (h1) 
    (h2) edge [loop above] node {$\top$} (h2) 
    (h1) edge  node {$\textit{g} \wedge \neg \textit{c} \wedge \neg \textit{p}$} (h2);
    
  \end{tikzpicture}
  
  \caption{An LLM-generated DFA for our motivating example.}
  \label{fig:advice_dfa_example}
\end{figure}

In our proposed method we generate the domain-specific knowledge using the GPT series \cite{liu2023summary}. There are several methods of prompting LLMs such as zero-shot method \cite{kojima2022large} where the prompt is using a \textit{step-by-step} approach to obtain better results or the few-shot method \cite{brown2020language} where the prompt is given similar examples encouraging the LLMs to produce similar reasoning to reach the correct answer. 
We use a combination of these methods to encourage the GPT to produce the correct information. First, we adopt a persona for the domain that we plan to extract relevant information from; in our example, the persona is an expert in American road traffic rules. Adopting a persona allows for more coherent and higher-quality responses from the language model. Figure \ref{fig:llm_traffic_rule_persona} demonstrates the GPT response to a question for making a right turn at a traffic light while an expert persona is adopted.

\begin{figure}
    \centering
    \includegraphics[scale=0.18]{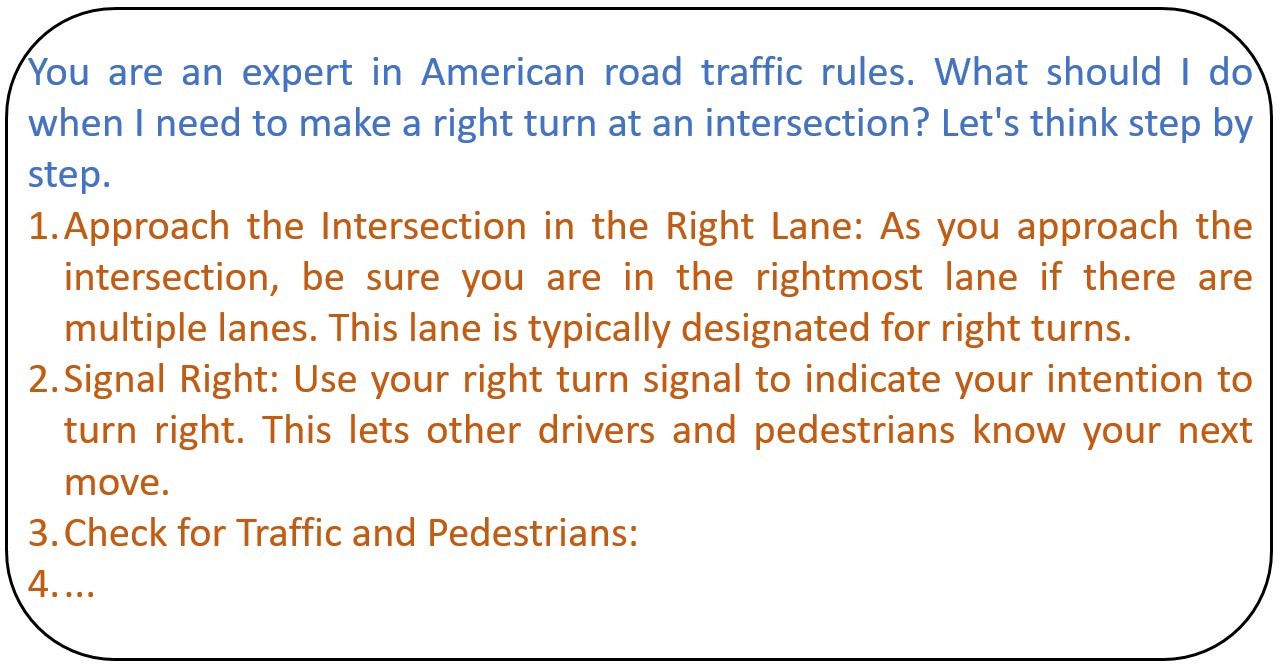}
    \caption{Prompting \texttt{GPT-3.5-Turbo} to adopt a persona as an expert in traffic rules (\textcolor[HTML]{c55a11}{GPT response} and \textcolor[HTML]{4472c4}{prompt}).}
    \label{fig:llm_traffic_rule_persona}
\end{figure}

Figure \ref{fig:llm_traffic_rule_persona} demonstrates that we can obtain the knowledge for a specific domain using the appropriate prompting method. In general, LLMs are text-based which means that their output is text and it cannot directly be used in an RL algorithm; therefore we require a method to convert this textual output to an automata. The next step is to ask the LLM to convert the provided steps to relevant propositions $\mathcal{P}$ so it can be used for the generation of the automata. We can perform this by mapping the steps into the set of propositions as illustrated in Figure \ref{fig:map_to_prop}.

\begin{figure}
    \centering
    \includegraphics[scale=0.2]{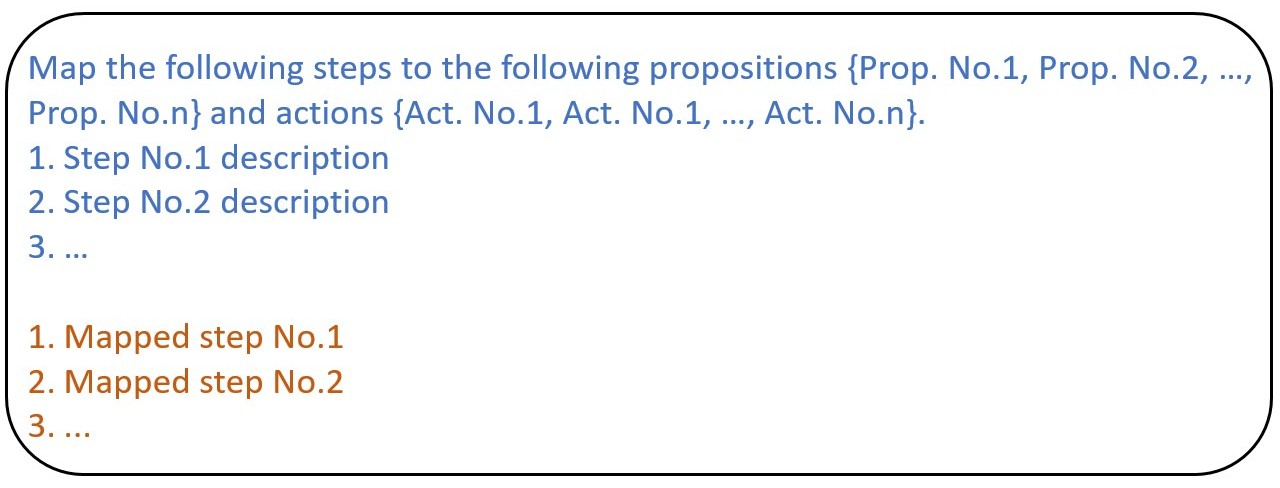}
    \caption{Mapping the output of the LLM to a specific set of propositions.}
    \label{fig:map_to_prop}
\end{figure}

We use the mapped steps (instructions) to generate the DFA since actions and propositions are available; however, transitions and states are yet to be defined. We can use the labels for transitions between the states and to obtain the states we can circle back to the task and use the instructions to obtain the states. Therefore, we use the task description and the mapped steps to obtain the transitions and states, Figure \ref{fig:state_transition} shows the prompting used to create the states and transitions.

\begin{figure}
    \centering
    \includegraphics[scale=0.2]{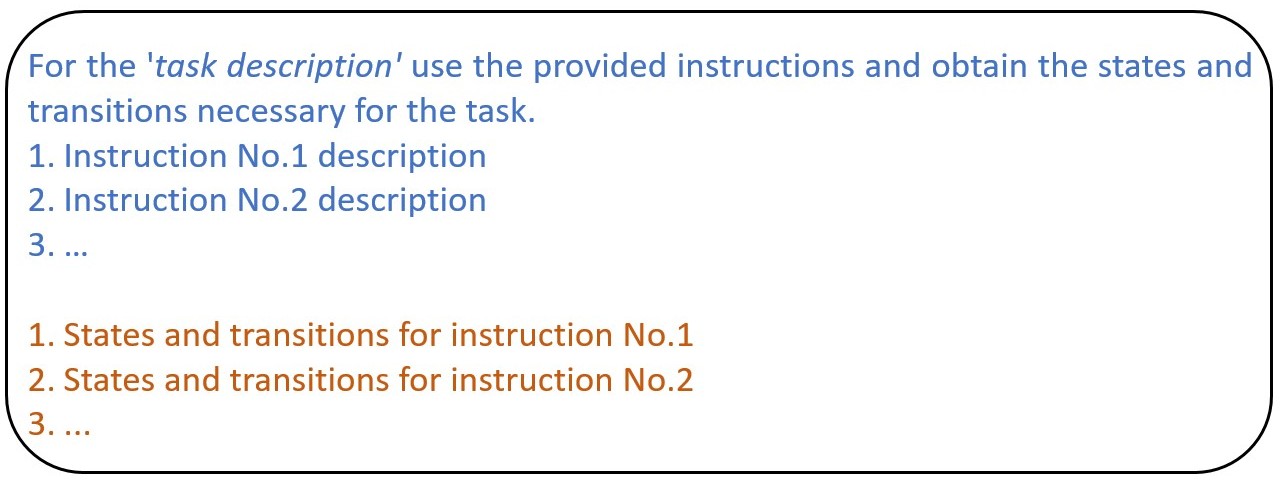}
    \caption{We use the mapped instructions to propositions and the task description to obtain the states and transitions.}
    \label{fig:state_transition}
\end{figure}

We can construct a DFA by having the states $h_{i}$, propositions $\mathcal{P}$, and transition function $\delta(h_{i}, l_{i})$; therefore, we can now use this information to generate the DFA for the task and use it to expedite convergence to optimal policy $\pi^{*}(s, a)$.

\begin{algorithm}[tb]
    \caption{\texttt{PromptLLMforDFA}: Constructing task DFA for a domain specific task using LLM }
    \label{alg:alg_prompt}
    \textbf{Input}: prompt $f$\\
    \textbf{Parameter}: temperature $T$, Top p $p$ (OpenAI LLM parameters)\\
    \textbf{Output}: DFA
    \begin{algorithmic}[1] 
        \STATE $output \gets \texttt{promptLLM}(f)$ 
        \FOR{$step \in output$}
        \STATE $M_{\text{prop}} \gets \texttt{MapToProp}(step)$
        \STATE $W_{\text{action}} \gets \texttt{MapToActions}(step)$
        \ENDFOR
        \STATE $output \gets \texttt{UpdatePrompt}(f, steps, M_{\text{prop}}, W_{\text{action}})$
        \STATE $H \gets \{h_0\}$
        \STATE $\delta \gets \{ \delta_{0}\}$
        \STATE $\mathcal{P} \gets \{ \mathscr{P}_{0}\}$
        
        \STATE $i \gets 0$
        \FOR{$instruction$ in $output$}
            \STATE $i \gets i+1$
            \STATE $h_{i} \gets \texttt{GetStates}(instruction)$
            \STATE $\delta_{i} \gets \texttt{GetTransition}(instruction)$
            \STATE $\mathscr{P}_{i} \gets \texttt{GetProposition}(instruction)$
            \STATE $ H \gets H \cup \{h_{i}\}$
            \STATE $ \delta \gets \delta \cup \{\delta_{i}\}$
            \STATE $ \mathcal{P} \gets \mathcal{P} \cup \{\mathscr{P}_{i}\}$
        \ENDFOR
        \STATE \textbf{return} $\langle H, h_{0}, 2^\mathcal{P}, \delta, \{h_{i}\} \rangle$
    \end{algorithmic}
\end{algorithm}

Algorithm \ref{alg:alg_prompt} can be updated to not even require propositions $\mathcal{P}$ and actions $A$ to obtain the DFA. Generating DFA this method requires a different prompting technique which combines two methods of few-shots and chain-of-thought we call MixedR (mixed-reasoning). We exploit this approach to obtain tasks that are more involved since obtaining the propositions and actions requires more information about the system; however, the MixedR requires the task description and examples of DFAs which do not require to even be relevant to the task as long as the descriptions are provided for the example DFAs even in high-level
since the LLMs are capable of extracting complex patterns from textual data. We use MixedR to generate the DFA. Figure \ref{fig:mixR_prompt} shows the prompt template used for MixedR.

\begin{figure}[H]
    \centering
    \includegraphics[scale=0.2]{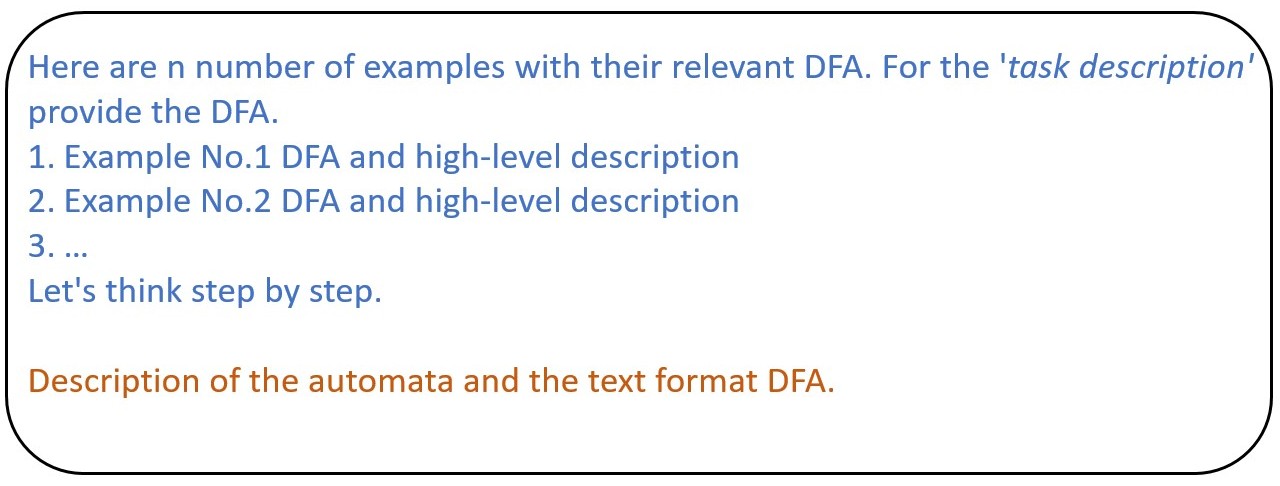}
    \caption{We use the MixedR to obtain the DFAs that are more complex.}
    \label{fig:mixR_prompt}
\end{figure}

We can use the motivating example and the MixedR prompt technique to obtain the relevant DFA for a smaller example of the problem. We can then use the DFA obtained from the MixedR and provide it as LLM-generated DFA for the RL algorithm to use in order to learn the ground truth reward machine in a smaller search space.

\section{Expediting Reinforcement Learning Using LLM-generated DFA}
We use the LLM output as instructions to expedite the reinforcement learning convergence to the optimal policy $\pi^{*}(s, a)$ using the language models. We use the LLMs to narrow down the search space for the RL algorithm since in a typical reinforcement learning algorithm the search space is considerable leading to cases where the training requires significant time and resources. Hence, using the pre-trained LLMs we can reduce this search space and converge to the optimal policy faster. One main issue to address is that LLMs typically generate outputs that are not aligned with facts; therefore, it is necessary to close the loop in the RL algorithm so that if a counterexample is met then the algorithm should be capable of adjusting the prompt, and updating it to create an updated LLM-generated DFA that could be used by the RL algorithm. Figure \ref{fig:alg_overview} illustrates the LARL-RM algorithm.

\begin{figure}
    \centering
    \includegraphics[scale=0.25]{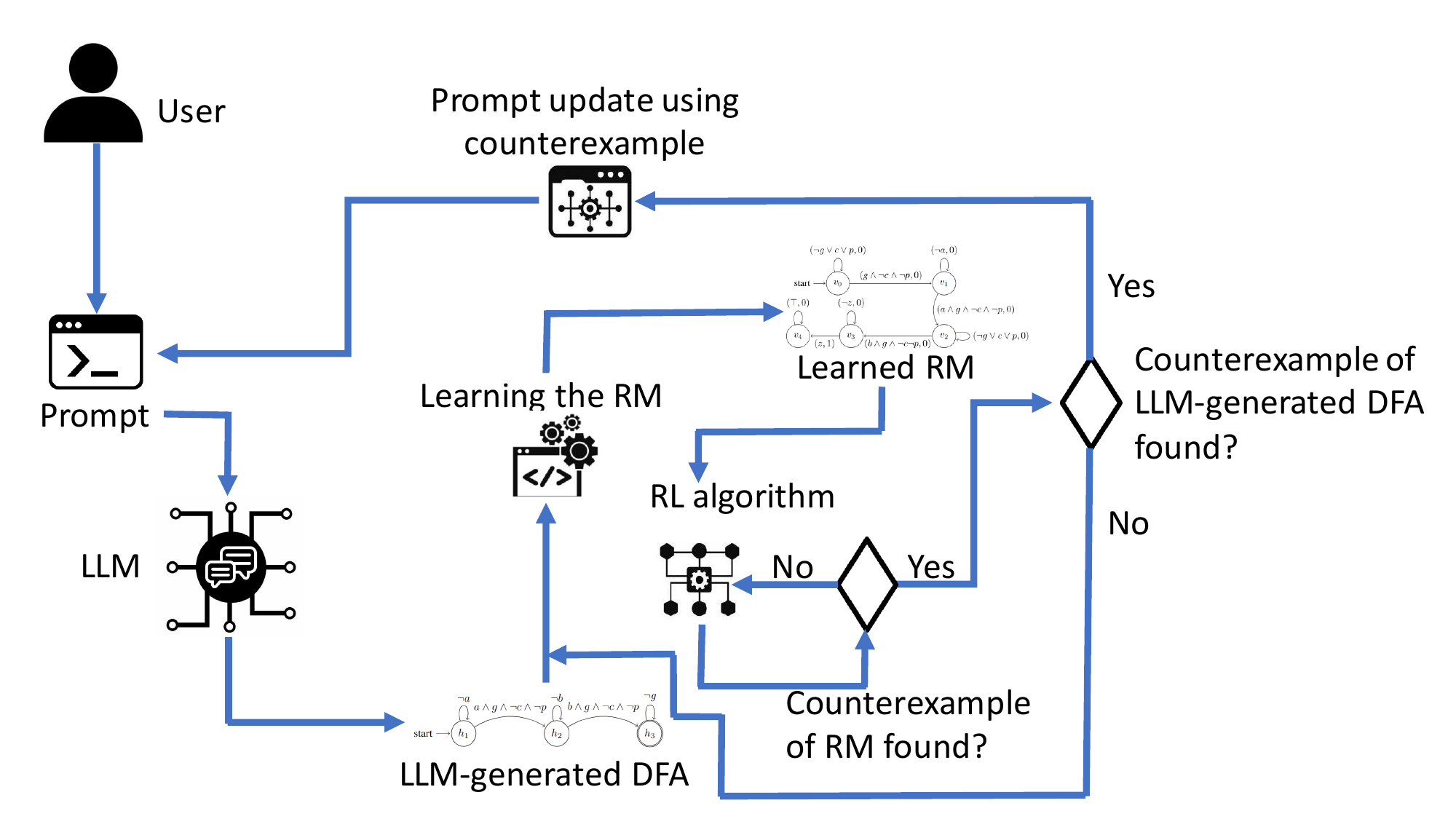}
    \caption{LARL-RM uses the LLM to generate the automata. The LLM-generated automata will be used by the RL algorithm to expedite the convergence to optimal policy.}
    \label{fig:alg_overview}
\end{figure}

\begin{algorithm}[tb]
    \caption{LARL-RM algorithm for incorporating high-level domain-specific knowledge from LLM into RL}
    \label{alg:LARL-RM_alg}
    \textbf{Input}: prompt $f$, Episode length $Episode_{length}$ \\
    \textbf{Parameter}: learning rate $\alpha$, discount factor $\gamma$, epsilon-greedy $\epsilon$, temperature $T$, Top p $p$, LLM query budget $\mathcal{J} \in \mathbb{N}$ \\
    \begin{algorithmic}[1]
        \STATE $X \gets \emptyset$ (empty sample) \label{line:emptyset}
        \STATE $\mathcal{D} \gets \texttt{PromptLLMforDFA}(f)$ \label{line:firstprompt}
        \STATE $\mathcal{A} \gets \texttt{InitializeRewardMachine}()$ (compatible with $\mathcal{D}$)                        \label{line:initializeRM}
        \STATE $q(s, v, a) \gets \texttt{InitializeQFunction}()$, $Q = \{ q^{q} | q \in V \}$                                \label{line:initQfun}
        \FOR{$episode$ in $1,\dots,Episode_{length}$} \label{line:episodeStart}
        \STATE $X_{\text{init}} \gets  X$ \label{line:saveTrace}
        \STATE $\mathcal{D}_{\text{init}} \gets  \mathcal{D}$ \label{line:saveDFA}
        \STATE $(\lambda, \rho, Q) \gets \texttt{QRM-episode} (\mathcal{A}, Q)$ \label{line:qleanring}
        \IF{$\mathcal{A}(\lambda) \neq \rho$} \label{line:counterExampleRM}
        \STATE $X \gets X \cup \{(\lambda, \rho)\}$ \label{line:RMUpdateTrace}
        \ENDIF
        \IF{$\mathcal{D} \neq \mathcal{D}_{\text{init}}~\text{or} X \neq X_{\text{init}} $} \label{line:checkUpdatedLists}
            \label{line:checkCOnpatibility} \IF{$\mathcal{J} \ge 0$ and $\exists (\lambda',\rho') \in X, \ \rho'>0 \text{ and } \lambda'\notin \Language(\mathcal{D})$} \label{line:DFAcounter}
                \STATE $f \gets \texttt{UpdatePrompt}(f, \lambda')$ \label{line:promptUpdate}
                \STATE $\mathcal{D} \gets \texttt{PromptLLMforDFA}(f)$ \label{line:PromptAgain}
                \STATE $\mathcal{J} \gets \mathcal{J} - 1$ \label{line:budget_reduce}
            \ENDIF
            \STATE $\mathcal{A} \gets \texttt{LearnRewardMachine}()$ (compatible with $\mathcal{D}$ and $X$) \label{line:RelearnRM}
            \STATE $q(s, v, a) \gets \texttt{InitializeQFunction}()$  \label{line:ReintQ}
        \ENDIF
        \ENDFOR
    \end{algorithmic}
\end{algorithm}

Our proposed method LARL-RM uses the prompt provided by the user to run the RL algorithm for a specific task that requires domain knowledge; however, by using the LLM the need for an expert is minimized and the algorithm itself can update the prompt using the counterexamples to update its DFA. LARL-RM first initializes an empty set for the trace $X$, then uses the prompt to generate the relevant DFA for the RL (Lines \ref{line:emptyset}-\ref{line:firstprompt}). The algorithm then initializes the reward machine based on the LLM-generated DFA obtained from the LLM and also initializes the QRM (q-learning for reward machines) before starting the episode (Lines \ref{line:initializeRM}-\ref{line:initQfun}). LARL-RM stores the trace $X$ and DFA $\mathcal{D}$ in order to reinitialize the reward machines and q-values if counterexamples are met, then the QRM is called to update the q-values and if there are any counterexamples then that trace is stored so that it can be used to update the reward machine and the DFA (Lines \ref{line:episodeStart}-\ref{line:RMUpdateTrace}). If the trace met by the agent is not compatible with the ground truth reward machine then it is removed from the LLM-generated DFA set $\mathcal{D}$ and then the algorithm uses this incompatible DFA and the initial prompt $f$ to obtain an updated prompt $f$ using the counterexample label sequence

, LARL-RM also uses an LLM query budget $\mathcal{J}$ to ensure that if the responses are not compatible with the ground truth reward machine then it will not get stuck in a loop and eventually after the budget $\mathcal{J}$ is depleted then it can start to learn the ground truth reward machine without LLM-generated DFA  (Lines: \ref{line:DFAcounter}-\ref{line:budget_reduce}, this also highlights the importance of prompt engineering further). Afterward, LARL-RM uses the stored DFA and trace sets to reinitialize the reward machine and q-values (Lines \ref{line:RelearnRM}-\ref{line:ReintQ}). 

\subsection{Using LLM-generated DFA to Learn the Reward Machine}
LARL-RM uses the counterexamples to create a minimal reward machine while being guided by the LLM-generated DFA. Our method uses a similar approach as of the \cite{neider2021advice} such that it maintains a finite sample set $X \subset 2^{\mathcal{P}} \times \mathbb{R}$ and a finite set of DFAs $\mathscr{D} = \{ \mathcal{}{D}_{1}, \ldots, \mathcal{D}_{l}\}$. We assume that the LLM-generated DFA is compatible with the sample $X$ such that the $(l_{1} \ldots l_{k}, r_{1} \ldots r_{k}) \in X$ with positive reward $r_{k} > 0$ where $l_{1} \ldots l_{k} \in \Language(\mathcal{D})$ and if this criterion is not fulfilled then it gets ignored by the LARL-RM.

The RM learning which relies on the LLM-generated DFA performs the generation of the minimal reward machine $\mathcal{A}$ which is consistent with the trace $X$ and compatible with $\mathcal{D}$. The minimal requirement for the reward machine to be minimal is important to the convergence of the LARL-RM to the optimal policy. LARL-RM uses SAT-based automata learning to verify the parametric systems since the learning task can be broken down into a series of satisfiability checks for propositional logic formulas. Essentially we construct and solve the propositional formulas $\Phi^{X, \mathcal{D}}_{n}$ where the values of $n>0$ \cite{neider2021advice}\cite{neider2014applications}.  

We construct the formula $\Phi^{X, \mathcal{D}}_{n}$ based on a propositional set $ X = \{ x, y, z, \ldots \}$ using the Boolean connectives $\neg, \vee, \wedge$, and $\mapsto$, i.e., and interpretation is the mapping from propositional formulas to Boolean values such that $\mathcal{I}: X \mapsto \{0, 1\}$.
If the the propositional formula is satisfied then $\mathcal{I} \models \Phi$ which interprets to $\mathcal{I}$ satisfies formula $\Phi$.
Using this approach we can obtain a minimal reward machine similar to \cite{neider2021advice}.
Using the LLM we are capable of constructing an LLM-generated DFA in a format that is compatible with SAT-based automata learning methods.
LARL-RM is also capable of adjusting the initial prompt $f$ so that in case it is not compatible with the ground truth reward machine then it could be updated $f$.
Updated prompt $f$ uses the counterexample label sequence $\lambda'$ to call the LLM and obtain an updated DFA $\mathcal{D}$ which is compatible with the ground truth reward machine.

\subsection{Refinement of Prompt and DFA}

We use prompt $f$ to generate the DFA $\mathcal{D} = \langle H, h_{I}, \Sigma, \delta, F \rangle $ for a specific domain.
Transition exists for a proposition if $\delta(h_{i}, l_{i}) = 1$, meaning if the LLM-generated DFA generated by the prompt $f$ is incorrect then it cannot have the correct transitions and trajectory which leads to a counterexample $\lambda'$.

Therefore, we use the counterexample $\lambda' = l_{1} l_{2} \ldots l_{k}$ to update the prompt in order to generate an updated DFA.

Hence, the updated DFA has a chance of becoming compatible with the counterexample. This process continues and each time the algorithm encounters a counterexample ($\lambda'$) uses it to update the prompt $f$ again.

\section{Convergence to Optimal Policy}
The learned reward machine will ultimately be equivalent to the ground truth reward machine under the condition that any label sequence is admissible by the underlying MDP, i.e., that are physically allowed.
If this assumption does not hold, it is still ensured that learned reward machine and the ground truth reward machine will agree on every admissible label sequences.

Due to the fact that we use LLM to generate the DFA $\mathcal{D}$ and the fact that the LLMs are known to produce outputs that are not factual, we face an issue.
We cannot make any assumption on the quality of the output of the LLM.
For the purpose of proving guaranties of LARL-RM, we need to consider the LLM as adversarial, that is, considering the worst case.

\begin{lemma} \label{lem:learn-correct-RM}
Let $\mathcal M$ be a labeled MDP, $\mathcal A$ the ground truth reward machine encoding the rewards of $\mathcal{M}$, and $\mathscr{D}^\star = \{\mathcal{D}_{1}, \ldots, \mathcal{D}_{m} \}$ the set of all LLM-generated DFAs that are added to $\mathcal D$ during the run of LARL-RM.
Additionally, let $n_{\text{max}} = \max_{\mathcal{D} \in \mathscr{D}^\star} \{ | \mathcal{D} | \}$ and $m = \max \left\{ 2 \lvert \mathcal{M} \rvert \cdot \left( \lvert \mathcal{A} \rvert +1 \right) \cdot n_{\text{max}}, \lvert \mathcal{M} \rvert \left( \lvert \mathcal{A} \rvert +1 \right)^{2} \right\}$.
Then, LARL-RM with $Episode_{length} \ge m$ almost surely learns a reward machine that is equivalent to $\mathcal{A}$.
\end{lemma}

LARL-RM provides us with an upper bound for the episode length that needs to be explored. Additionally, LARL-RM algorithm correctness follows the Lemma \ref{lem:learn-correct-RM} and the correctness of QRM algorithm \cite{icarte2018using}. Using the QRM algorithm guarantee we can now show the convergence to optimal policy

\begin{theorem}\label{th:optimal_policy}
Let $\mathcal M$, $\mathcal A$, $\mathscr{D}^\star$, and $m$ be as in Lemma~\ref{lem:learn-correct-RM}.
Then, LARL-RM will converge to an optimal policy almost surely if $Episode_{\text{length}} \ge m$.
\end{theorem}

Theorem \ref{th:optimal_policy} guarantees the convergence of the LARL-RM to an optimal policy if sufficient episode length is given for exploration. It also provides an upper bound for convergence as shown in Lemma \ref{lem:learn-correct-RM}.

\section{Case Studies}
We implement the LARL-RM using the GPT series LLM, specifically we focus on model \texttt{GPT-3.5-Turbo}. We can set up a chat with either of the models using the provided APIs from OpenAI. For our example we consider the LARL-RM applied to an autonomous car example similar to the example found in \cite{xu2020joint}. 

In our example, the autonomous car must navigate to reach a destination while avoiding any pedestrians, but obeying traffic laws. It is worth mentioning that our algorithm is capable of running even without any LLM-generated DFA since in this case it can learn the ground truth reward machine without LLM-generated DFA. We consider two case studies for our motivating example to demonstrate the LARL-RM capabilities to expedite the RL. In both cases, we demonstrate the effect of incompatible LLM-generated DFA (not compatible with the ground truth reward machine) on the algorithm convergence to the optimal policy.

\subsection{Case Study 1}
The agent in our traffic example has the following set of actions $A=\{\textit{up}, \textit{down}, \textit{right}, \textit{left}, \textit{stay} \}$. The layout of the environment is shown in Figure \ref{fig:environment}

We use the LLM-generated DFA $\mathcal{D}$ to guide the RL process. If the advice from the LLM is not compatible with the ground truth reward machine then it will be ignored by the algorithm. Prompt is updated and fed back to the LLM in order to obtain a new DFA that has a chance of becoming compatible with the counterexample. We show the ground truth reward machine for Case Study 1 in Figure \ref{fig:case_1_grnd_rm}.

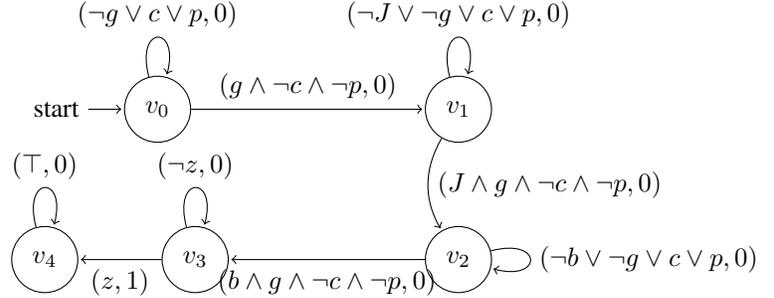
\begin{figure}[H]
  \centering
  \begin{tikzpicture}[shorten >=1pt, node distance=2.cm, on grid, auto]
    \node[state, initial] (v0) {$v_0$};
    \node[state] (v1) [right= 4cm of v0] {$v_1$};
    \node[state] (v2) [below=of v1] {$v_2$};
    \node[state] (v3) [left=3.5cm of v2] {$v_3$};
    \node[state] (v4) [left=of v3] {$v_4$};

    \path[->]
    (v0) edge   node {$(g \wedge \neg c \wedge \neg p, 0)$} (v1)
    (v0) edge [loop above] node {$(\neg g \vee c \vee p, 0)$} (v0)
    (v1) edge [loop above] node {$(\neg J \vee \neg g \vee c \vee p, 0)$} (v1)
    (v1) edge [bend right] node {$(J \wedge g \wedge \neg c \wedge \neg p, 0)$} (v2)
    (v2) edge [loop right] node {$(\neg b \vee \neg g \vee c \vee p, 0)$} (v2)
    (v2) edge  node {$(b \wedge g \wedge \neg c \wedge \neg p, 0)$} (v3)
    (v3) edge [loop above] node {$(\neg z, 0)$} (v3)
    (v3) edge  node {$(z, 1)$} (v4)
    (v4) edge [loop above] node {$(\top, 0)$} (v4);

  \end{tikzpicture}
  
  \caption{Ground truth reward machine for Case Study 1. Agent must first navigate to intersection $J$ then $b$.}
  \label{fig:case_1_grnd_rm}
\end{figure}

The LLM-generated DFA generated by the prompt $f$ for this case study is compatible with the ground truth reward machine (Figure \ref{fig:case_1_grnd_rm}) and helps the algorithm to converge to the optimal policy faster. The generated DFA using the GPT is shown in Figure \ref{fig:case_1_dfa}.

\begin{figure}[H]
  \centering
  \begin{tikzpicture}[shorten >=1pt, node distance=3cm, on grid, auto, state/.style={circle, draw, minimum size=2em}]
    \node[state, initial] (h1) {$h_1$};
    \node[state, accepting] (h2) [right=of h1] {$h_2$};

    \path[->]
    (h1) edge [bend left] node {$J \wedge g \wedge \neg c\wedge \neg p$} (h2)
    (h1) edge [loop above] node {$\lnot J \lor \lnot g \lor c \lor p$} (h1)
    (h2) edge [loop above] node {$\top$} (h2);
  \end{tikzpicture}
  
  \caption{LLM-generated DFA $\mathcal{D}, \{J \wedge g \wedge \neg c \wedge \neg p\}$ generated by LLM for Case Study 1.}
  \label{fig:case_1_dfa}
\end{figure}
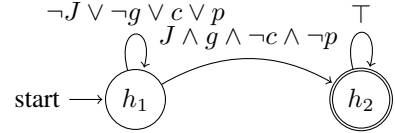

We use the DFA (\ref{fig:case_1_dfa}) to guide the learning process. The reward obtained by the LARL-RM algorithm shows that it can reach the optimal policy faster if the LLM-generated DFA exists in comparison to the case when there is none.

\begin{figure}[H]
    \centering
    \includegraphics[scale=0.27]{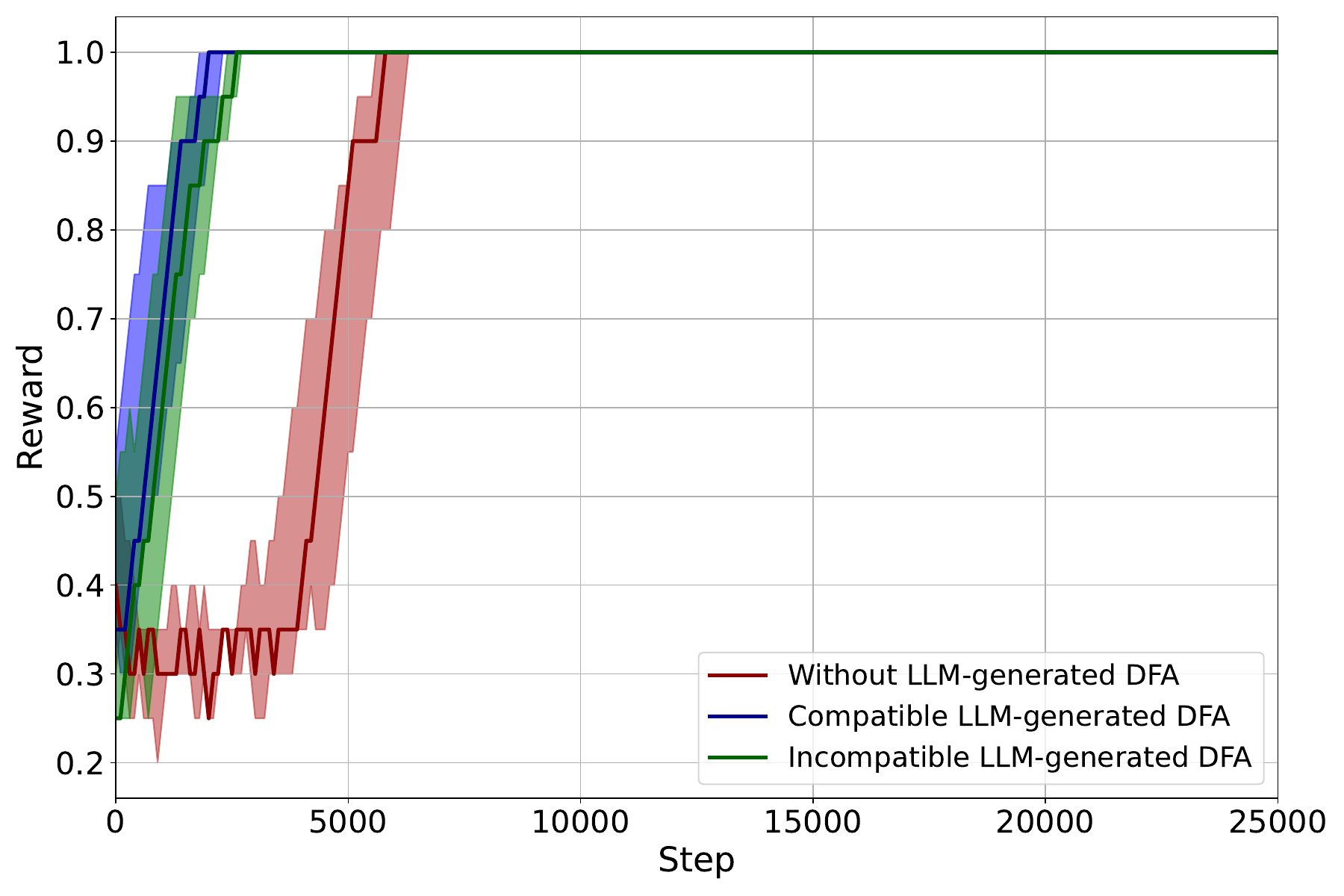}
    \caption{LARL-RM uses the LLM-generated DFA $\mathcal{D}$ to obtain a reward machine that is aligned with the ground truth reward machine in order to converge to an optimal policy faster.}
    \label{fig:rew_case_1}
\end{figure}

We average the results of $5$ independent runs for using a rolling mean with a window size of $10$. Figure \ref{fig:rew_case_1} demonstrates that the algorithm reaches the convergence policy using the LLM-generated DFA. 
In this example, the prompt $f$ is one time compatible and one time incompatible with the ground truth reward machine, but even if it is not compatible LARL-RM can use the counterexample to update the prompt and generate a new DFA which is more likely to be compatible with the counterexample.

\subsection{Case Study 2}
In Case Study 2 we show that if the suggested DFA $\mathcal{D}$ is not compatible with the ground truth reward machine then the LARL-RM uses the counterexample to adjust the prompt $f$ and update it to obtain new DFA which is compatible with the counterexample. In this case study we update the MDP environment to allow for multiple reward machines as well as DFAs to be considered as equivalent. This way the output of the LLM may not necessarily match the ground truth reward machine; however, the LARL-RM can still use this DFA as long as it is not incompatible with the ground truth reward machine, but is equivalent to it. 

In this environment, we extend the autonomous car example and provide it with more options to reach the destinations and the ground truth reward machine does not consider a specific route, but rather emphasizes the target destination; however, the LLM-generated DFA might specify a certain route over another which is not incompatible with the ground truth reward machine, but rather equivalent. We show the ground truth reward machine for Case Study 2 in Figure \ref{fig:case_2_grnd_rm}.

\begin{figure}
  \centering
  \begin{tikzpicture}[node distance=2.cm, on grid, auto,
                      state/.style={circle, draw, minimum size=2em}]
    \node[state, initial] (v0) {$v_0$};
    \node[state] (v1) [right=3.3cm of v0] {$v_1$};
    \node[state] (v2) [below=of v1] {$v_2$};
    \node[state] (v3) [below=of v2] {$v_3$};
    \node[state] (v4) [left=of v2] {$v_4$};
    \node[state] (v5) [below left=of v4] {$v_5$};
    
    \path[->]
    (v0) edge   node {$(g \wedge \neg c \wedge \neg p, 0)$} (v1)
    (v0) edge [loop above] node {$(\neg g \vee c \vee p, 0)$} (v0)
    (v1) edge [bend left] node {$(b \wedge g \wedge \neg c \wedge \neg p, 0)$} (v2)
    (v1) edge [loop above] node {$(\neg b \vee \neg g \vee c \vee p, 0)$} (v1)
    (v2) edge [loop right] node {$(\neg e \vee \neg g \vee p \vee c, 0)$} (v2)
    (v2) edge [bend right] node {$(e \wedge g \wedge \neg p \wedge \neg c, 0)$} (v3)
    (v3) edge [bend left] node {$(\mathcal{B}, 0)$} (v4) 
    (v4) edge [loop above] node {$(\neg z, 0)$} (v4)
    (v3) edge [loop right] node {$(\neg b \vee \neg g \vee p \vee c, 0)$} (v3)
    (v4) edge [bend right] node {$(z, 1)$} (v5)
    (v5) edge [loop left] node {$(\top, 0)$} (v5);
  \end{tikzpicture}
  
  \caption{Ground truth reward machine for Case Study 2. Agent must reach the destination no matter the route ($\mathcal{B} = b \wedge g \wedge \neg p \wedge \neg c$).}
  \label{fig:case_2_grnd_rm}
\end{figure}
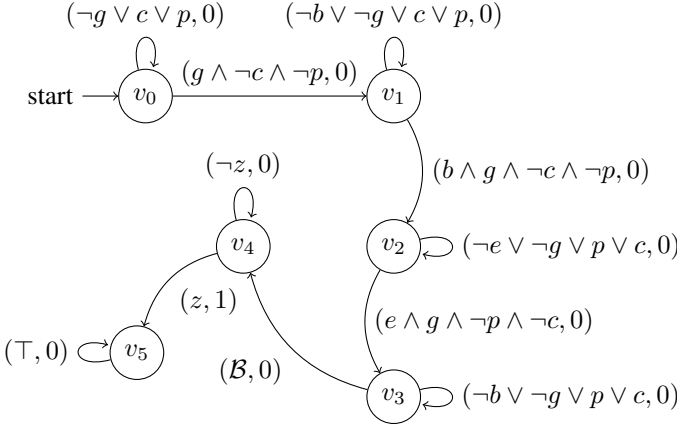

The LLM-generated DFA $\mathcal{D}$, $ \{J \wedge g \wedge \neg c \wedge \neg p, b \wedge g \wedge \neg c \wedge \neg p \}$ for this case study specifies a route which leads to the target and it is equivalent to the ground truth reward machine since this solution DFA can be considered a subset of the larger DFA $\mathcal{D}$ which is compatible with the ground truth reward machine. Figure \ref{fig:case_2_dfa} shows the LLM-generated DFA for Case Study 2.

\begin{figure}
  \centering
  \begin{tikzpicture}[shorten >=1pt, node distance=3cm, on grid, auto, state/.style={circle, draw, minimum size=2em}]
    \node[state, initial] (h1) {$h_1$};
    \node[state] (h2) [right=of h1] {$h_2$};
    \node[state, accepting] (h3) [right=of h2] {$h_3$};

    \path[->]
    (h1) edge [bend left] node {$J \wedge g \wedge \neg c \wedge \neg p$} (h2)
    (h2) edge [bend left] node {$b \wedge g \wedge \neg c \wedge \neg p$} (h3)
    (h1) edge [loop above] node {$\lnot J \lor \lnot g \lor c \lor p$} (h1)
    (h2) edge [loop above] node {$\lnot b \lor \lnot g \lor c \lor p$} (h2)
    (h3) edge [loop above] node {$\top$} (h3);
  \end{tikzpicture}
  
  \caption{LLM-generated DFA $\mathcal{D}$ contains $\{b \wedge g \wedge \neg c \wedge \neg p\}$ generated by LLM for Case Study 2.}
  \label{fig:case_2_dfa}
\end{figure}
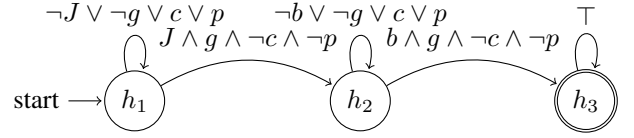

We run the LARL-RM for the Case Study 2 using the prompt $f$ which generates the DFA $\mathcal{D}$, $\{J, b \wedge g \wedge \neg c \wedge \neg p\}$ which is incompatible with the ground truth reward machine, but the LARL-RM takes this counterexample and generates an updated DFA $\mathcal{D}$, $\{ b \wedge g \wedge \neg c \wedge \neg p \}$ which is compatible with the ground truth reward machine. Figure \ref{fig:rew_case_2} demonstrates the reward of LARL-RM using the DFA $\mathcal{D}$.

\begin{figure}[H]
    \centering
    \includegraphics[scale=0.27]{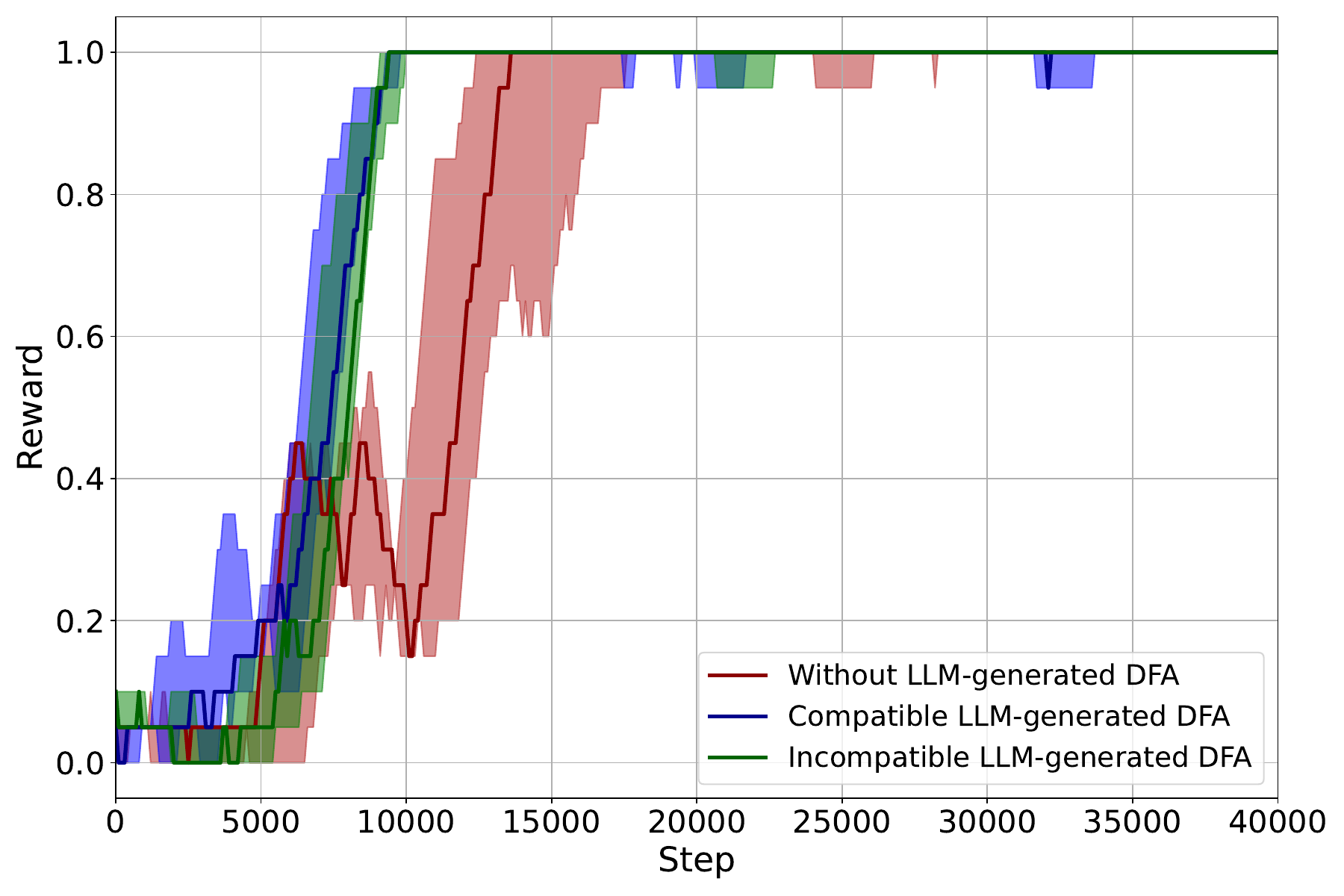}
    \caption{LARL-RM updates the prompt $f$ in order to obtain a DFA that is compatible with the ground truth reward machine.}
    \label{fig:rew_case_2}
\end{figure}

Figure \ref{fig:rew_case_2} shows the reward obtained by LARL-RM for Case Study 2 using the prompt $f$ which generates the DFA $\mathcal{D}$. The reward is for $5$ independent runs averaged at each $10$th step. For both cases of compatible and incompatible LLM-generated DFA the LARL-RM converges to the optimal policy faster than when there is no LLM-generated DFA.

\section{Conclusion}

We proposed a novel algorithm, LARL-RM, that uses a prompt to obtain an LLM-generated DFA to expedite reinforcement learning. LARL-RM uses counterexamples to automatically generate a new prompt and consequently, a new DFA that is compatible with the counterexamples to close the loop in RL. We showed that LARL-RM is guaranteed to converge to an optimal policy using the LLM-generated DFA. We showed that RL can be expedited using the LLM-generated DFA and in case the output of the LLM is not compatible with the ground truth reward machine, LARL-RM is capable of adjusting the prompt to obtain a more accurate DFA. In future work, we plan to extend the proposed framework to RL for multi-agent systems. 

\newpage

\section*{Acknowledgments}

This research is partially supported by the National Science Foundation under grant NSF CNS 2304863 and the Office of Naval Research under grant ONR N00014-23-1-2505.

\bibliographystyle{named}
\bibliography{ijcai24}

\appendix
\newpage

\section*{Supplementary Materials}

\section{Comparison of closed-loop and Open-loop LARL-RM}
Our proposed algorithm can be modified to perform in an open-loop or closed-loop manner with respect to LLM-generated DFA, meaning that if a counterexample is found \algname{} has the capability to either update the LLM-generated DFA $D$ or keep the original LLM-generated DFA.
LARL-RM has two options either continuing the learning process of the reward machine without updating the LLM-generated DFA, referred to as open-loop.
It can also call the LLM again and generate a new DFA based on the counterexample, referred to as the closed loop. We show our proposed algorithm's convergence to optimal policy under both options. We demonstrate the open-loop configuration in Figure \ref{fig:open_loop_config}.

\begin{figure}[H]
    \centering
    \includegraphics[scale=0.27]{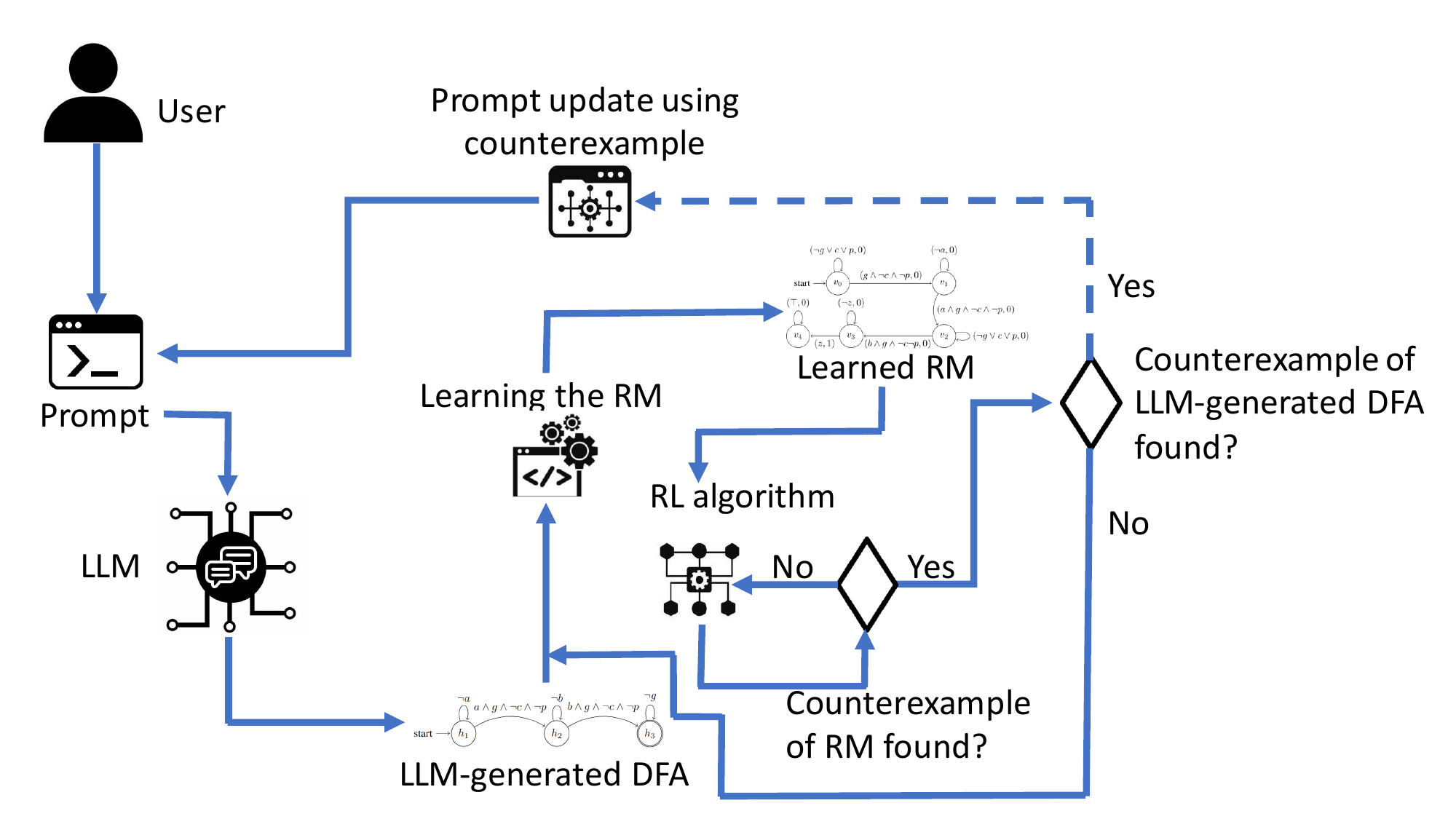}
    \caption{LARL-RM open-loop configuration, counterexample of LLM-generated DFAs are not used to update the prompt (dashed arrow).}
    \label{fig:open_loop_config}
\end{figure}

\subsection{Case Study 1}
We investigate the two open-loop and closed-loop methods in the Case Study 1. First, we demonstrate the effect of incompatible LLM-generated DFA with compatible LLM-generated DFA in open-loop conditions. Figure \ref{fig:open_close_1_incompat} illustrates the convergence to optimal policy for open-loop configuration.

\begin{figure}[H]
    \centering
    \includegraphics[scale=0.27]{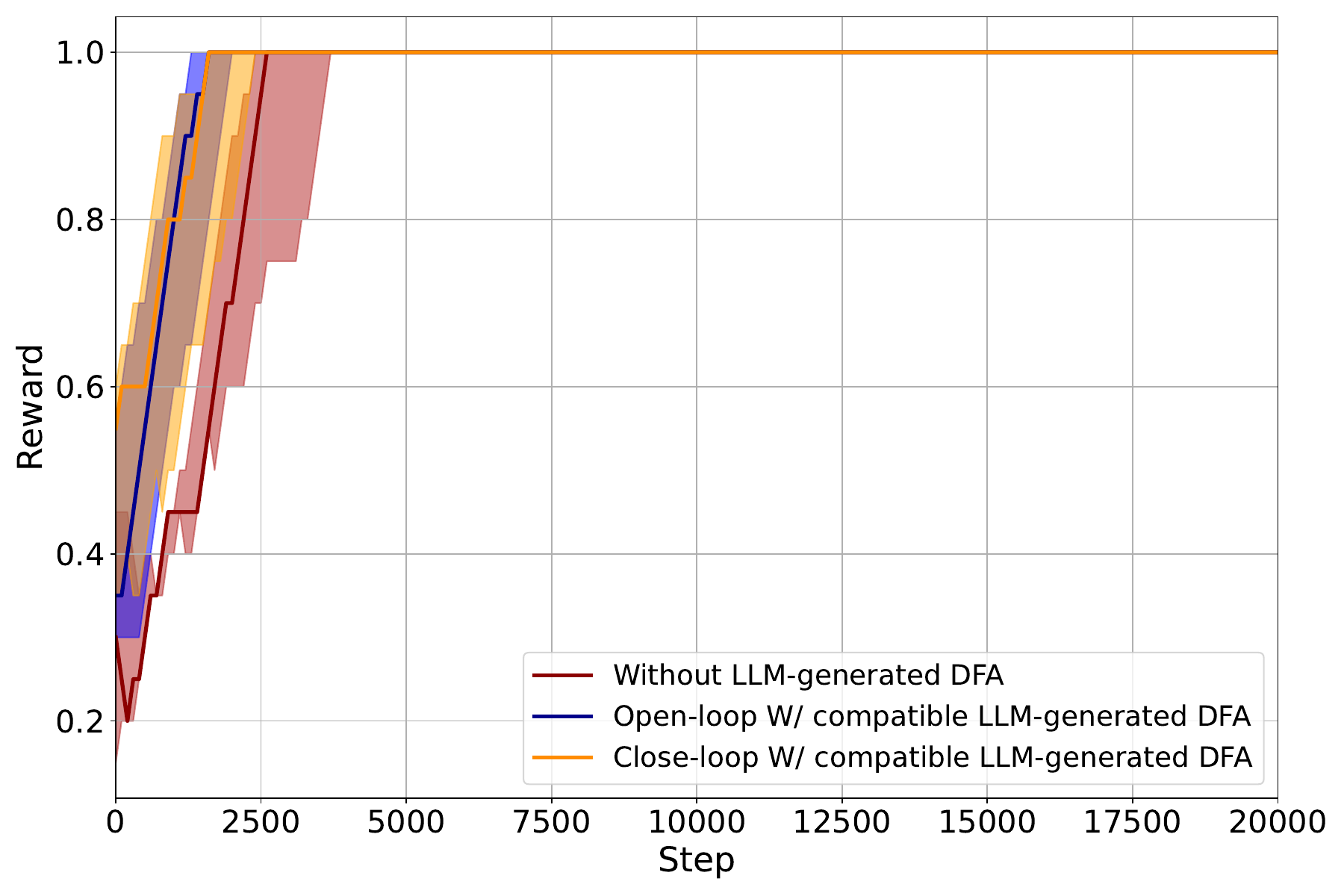}
    \caption{LARL-RM open-loop and closed-loop configurations' averaged rewards for Case Study 1 when the LLM-generated DFA is compatible with the ground truth reward machine. Rewards are for 5 independent runs, averaged at each $20$ step.}
    \label{fig:open_close_1_incompat}
\end{figure}

Figure \ref{fig:open_close_1_incompat} demonstrates that in LARL-RM open-loop configuration if the LLM-generated DFA is compatible it converges to the optimal policy faster than when there is no LLM-generated DFA.
However, if the LLM-generated DFA is incompatible then it might take longer to converge to the optimal policy.
Now we consider the closed-loop configuration so that the DFA if incompatible then be used to update the prompt $f$. Figure \ref{fig:open_close_1_compat} demonstrates the LARL-RM convergence to the optimal policy when the closed-loop configuration is applied.

\begin{figure}[H]
    \centering
    \includegraphics[scale=0.27]{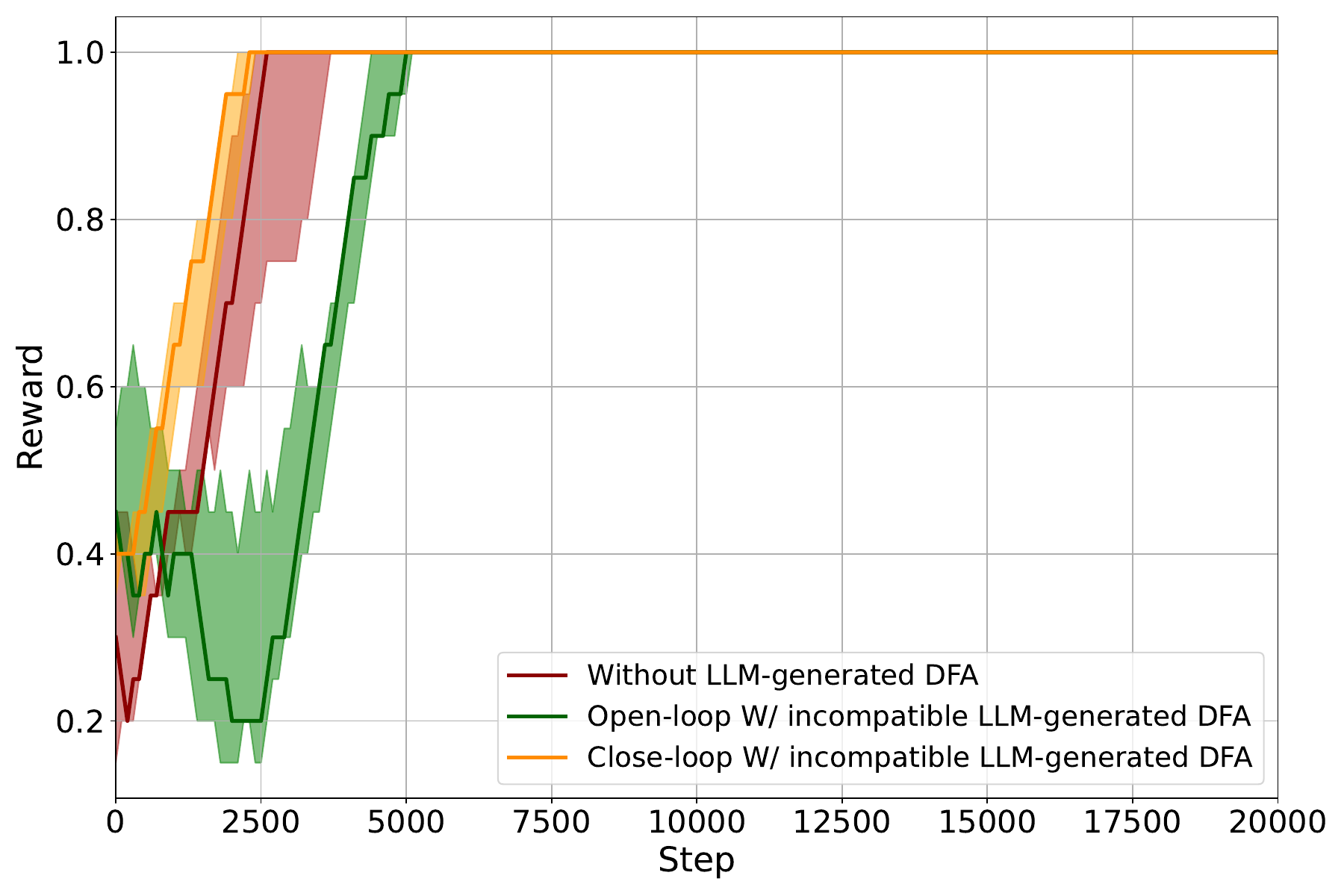}
    \caption{LARL-RM open-loop and closed-loop configurations' averaged rewards for Case Study 1 when the LLM-generated DFA is incompatible with the ground truth reward machine. Rewards are for 5 independent runs, averaged at each $20$ step.}
    \label{fig:open_close_1_compat}
\end{figure}

Figure \ref{fig:open_close_1_compat} demonstrates that the closed-loop configuration is capable of updating the prompt $f$ considering the counterexample such that LARL-RM has a better chance of converging faster than when there is no LLM-generated DFA.

\subsection{Case Study 2}
We further investigate the effect of open-loop and closed-loop configurations. The same configuration for the open-loop as illustrated in Figure \ref{fig:open_loop_config} is applied here.
The open-loop configuration as shown in Figure \ref{fig:open_close_2_incompat} is converging to the optimal policy faster than the case with no LLM-generated DFA if the LLM-generated DFA is compatible with the ground truth reward machine.

\begin{figure}[H]
    \centering
    \includegraphics[scale=0.27]{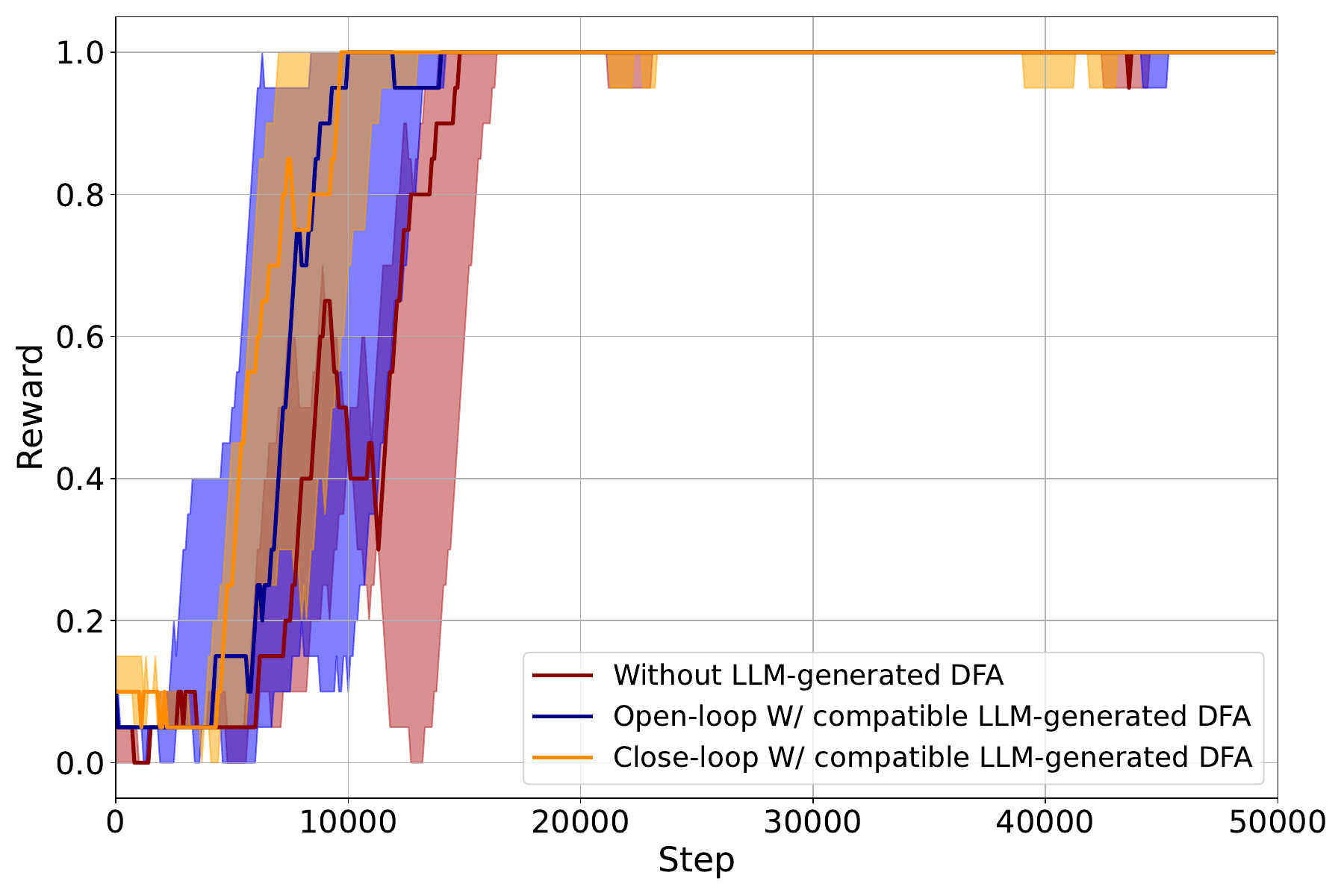}
    \caption{LARL-RM open-loop and closed-loop configurations reward for Case Study 2 when the LLM-generated DFA is compatible with the ground truth reward machine. Rewards are for 5 independent runs, averaged at each $20$ step.}
    \label{fig:open_close_2_incompat}
\end{figure}

If we apply the closed-loop configuration to Case Study 2, then as expected it can converge to the optimal policy faster than when there is no LLM-generated DFA and even if the initial LLM-generated DFA is incompatible with the ground truth reward machine.
Figure \ref{fig:open_close_2_compat} demonstrates the closed-loop configuration.

\begin{figure}[H]
    \centering
    \includegraphics[scale=0.27]{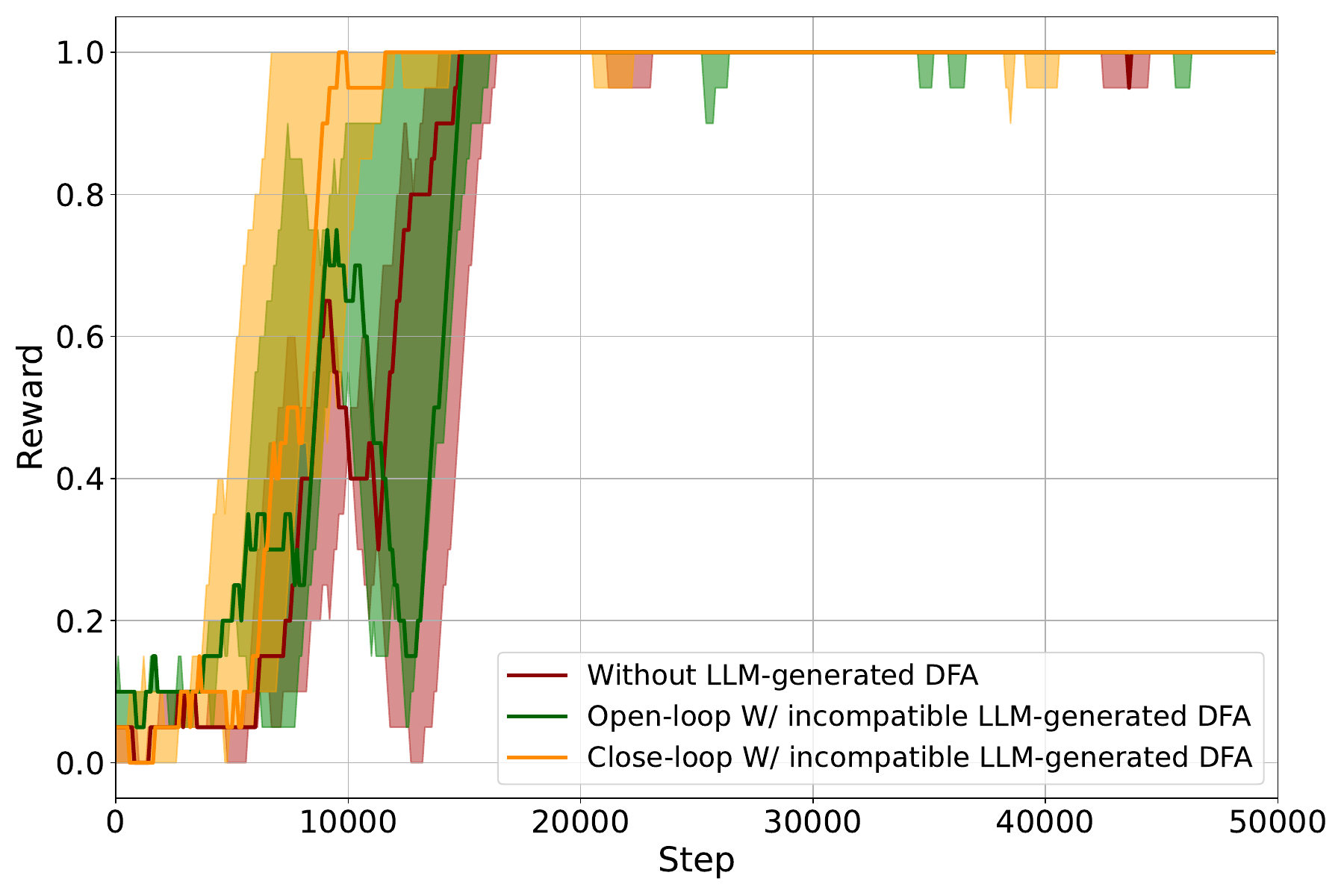}
    \caption{LARL-RM open-loop and closed-loop configurations reward for Case Study 2 when the LLM-generated DFA is incompatible with the ground truth reward machine. Rewards are for 5 independent runs, averaged at each $20$ step.}
    \label{fig:open_close_2_compat}
\end{figure}

As illustrated in Figure \ref{fig:open_close_2_compat}, if the closed-loop configuration is applied then the LARL-RM has a better chance to converge to the optimal policy faster than when there is no LLM-generated DFA.

\section{Theoretical Guarantee of the LARL-RM}
Here we demonstrate the convergence of the proposed algorithm to the optimal policy.

\begin{proof}[Proof of Lemma~\ref{lem:learn-correct-RM}]
We first show that the set $\mathscr{D}$ remains stationary after a finite amount of time (i.e., that no new LLM-generated DFA is added or removed), assuming that all trajectories will be visited infinitely often.
To this end, we observe the following: if an LLM-generated DFA added to $\mathscr{D}$ is compatible, then it will remain in $\mathscr{D}$ indefinitely.
The reason is that there is no counterexample contradicting the LLM-generated DFA, meaning that the check-in Line~\ref{line:checkCOnpatibility} is never triggered.
On the other hand, if an LLM-generated DFA is not consistent, then the algorithm eventually detects a trajectory witnessing this fact.
Once this happens, it removes the corresponding LLM-generated DFA from the set~$\mathscr{D}$.

Observe now that the algorithm decrements $\mathcal J$ by one every time it adds an LLM-generated DFA to $\mathscr{D}$, and it only does so as long as $\mathcal J > 0$.
Thus, the total number of LLM-generated DFA that are generated during the run of the algorithm is bounded by $\mathcal J$.
Consequently, the set $\mathscr{D}$ no longer changes after a finite period, because the algorithm has either identified the true reward machine or all incompatible LLM-generated DFAs, have been removed and the budget $J$ is exhausted.

Once $\mathscr{D}$ becomes stationary, an argument analogous to that in the proof of Lemma~1 in Neider et al.'s work~\cite{neider2021advice} shows that the algorithm will eventually learn the true reward machine.
Intuitively, as long as the current hypothesis is not equivalent to the true reward machine, there exists a ``short'' trajectory witnessing the fast.
The episode length is chosen carefully such that such a witness will eventually be encountered and added to $X$.
In the worst case, all trajectories of $Episode_{length}$ will eventually be added to $X$, at which point the learning algorithm is guaranteed to learn the true reward machine.
\end{proof}

Since our proposed algorithm learns the minimal reward machine that is consistent with $X$ and compatible with $D \in \mathscr{D}$, we need to use an automata learning method based on SAT which is used to verify parametric systems \cite{neider2014applications}.
The underlying idea is to reduce the learning task to a series of satisfiability checks of formulas for propositional logic.
In other words, we build and solve a sequence of propositional formulas $\Phi^{X, D}_{n}$ for $n > 0$ such that the following properties hold \cite{neider2021advice}:

\begin{itemize}
    \item the satisfiability of $\Phi_n^{X, D}$ is contingent upon the existence of a reward machine with $n$ states, such that it is consistent with $\mathcal X$ and compatible with  $\Advice \in \mathscr{D}$. This condition holds true if and only if such a reward machine exists.
    \item if $\Phi_n^{X, D}$ is satisfiable, it implies that a satisfying assignment contains enough information to construct a reward machine with $n$ states that is both consistent and compatible.
\end{itemize}

We can now build and solve $\Phi^{X, D}_{n}$ starting from $n=1$ and increasing $n$ until it becomes satisfiable. Using this method our proposed algorithm constructs a minimal reward machine that is consistent with $X$ and compatible with the LLM-generated DFA $D$.

To ease the notation in the remainder, we use $\machine \colon \mealyStateStyle{p} \run{w} \mealyStateStyle{q}$ to abbreviate a run of the reward machine $\machine$ on the input-sequence $w$ that starts in $\mealyStateStyle{p}$ and leads to $\mealyStateStyle{q}$.
By definition, we have $\machine \colon \mealyStateStyle{p} \run{\varepsilon} \mealyStateStyle{p}$ for the empty sequence $\varepsilon$ and every state $\mealyStateStyle{p}$.
We later also use this notation for DFAs.

To encode the reward machine $\machine = \langle V, V_{I}, 2^{\mathcal{P}}, R_{X}, \delta, \sigma \rangle$ in propositional logic where set $V$ of states and initial state $v_{I}$ is fixed, the reward machine $\machine$ is uniquely determined by the transitions $\delta$ and the output function $\sigma$.
The size of states $\lvert V \rvert = n$  and the initial state $v_{I}$ is fixed,
then to encode the transition function and the set representing the final
states, we will introduce two propositional variables namely $d_{p,l,v}$ where
$p, v \in V$ and $l \in 2^{\mathcal{P}}$ and $o_{p,l,r}$ for $p \in V$ , $l \in 2^{\mathcal{P}}$, and $r \in R_X$. We apply the constraints in \eqref{for:well-defined-1} and \eqref{for:well-defined-2} to ensure that the variables $d_{p, l, v}$ and $o_{p,l,r}$ encode deterministic functions.

\begin{align}
	\label{for:well-defined-1}
	\bigwedge_{\mealyStateStyle{p} \in \mealyStates} \bigwedge_{\mealyCommonInput \in \mealyInputAlphabet} \biggl[\! \bigl[ \bigvee_{\mealyStateStyle{q} \in \mealyStates} d_{\mealyStateStyle{p}, \mealyCommonInput, \mealyStateStyle{q}} \bigr] \land \bigl[ \bigwedge_{\mealyStateStyle{q} \neq \fsm{q'} \in \mealyStates} \lnot d_{\mealyStateStyle{p}, \mealyCommonInput, \mealyStateStyle{q}} \lor \lnot d_{\mealyStateStyle{p}, \mealyCommonInput, \mealyStateStyle{q'}} \bigr] \! \biggr] \\
	\label{for:well-defined-2}
	\bigwedge_{\mealyStateStyle{p} \in \mealyStates} \bigwedge_{\mealyCommonInput \in \mealyInputAlphabet} \biggl[\!\bigl[ \bigvee_{r \in \setOfSeenRewards_{X}} o_{\mealyStateStyle{p}, \mealyCommonInput, r} \bigr] \land \bigl[ \bigwedge_{r \neq r' \in \setOfSeenRewards_{X}} \lnot o_{\mealyStateStyle{p}, \mealyCommonInput, r} \lor \lnot o_{\mealyStateStyle{p}, \mealyCommonInput, r'} \bigr]\!\biggr]
\end{align}

We denote the conjunction of constraints \eqref{for:well-defined-1} and \eqref{for:well-defined-2} by $\Phi^\mathit{RM}_{n}$ which we later on use to show the consistency of the learned reward machine with $X$.
To apply constraints for consistency with samples in propositional logic, we propose auxiliary variables as $x_{\lambda, \rho}$ for every $(\lambda, \rho) \in \mathit{Pref}(X)$ where $q \in V$, and for a trace $\tau = (l_1 \ldots l_k, r_1 \ldots r_k) \in (2^\mathcal P)^\ast \times \mathbb R^\ast$ we define the set \emph{prefixes of $\tau$} by $\Pref(\tau) = \{ (l_1 \ldots l_i, r_1 \ldots r_i) \in (2^\mathcal P)^\ast \times \mathbb R^\ast \mid 0 \leq i \leq k \}$ ( $(\varepsilon, \varepsilon) \in \Pref(\tau)$ is always correct).
The value of $x_{\lambda, \rho}$ is set to true if and only if the prospective reward machine reaches states $q$ after reading $\lambda$.
To obtain the desired meaning, we add the following constraints:
\begin{align}
	\label{for:consistent-1}
	x_{\varepsilon, \fsm{q_I}} \land \bigwedge_{\mealyStateStyle{p} \in \mealyStates \setminus \{ \fsm{q_I} \}} \lnot x_{\varepsilon, \mealyStateStyle{p}} \\
	\label{for:consistent-2}
	\bigwedge_{(\lambda l, \rho r) \in \mathit{Pref}(X)} \bigwedge_{\mealyStateStyle{p}, \mealyStateStyle{q} \in \mealyStates} (x_{\lambda, p} \land d_{\mealyStateStyle{p}, l, \mealyStateStyle{q}}) \rightarrow x_{\lambda l, \mealyStateStyle{q}} \\
	\label{for:consistent-3}
	\bigwedge_{(\lambda l, \rho r) \in \mathit{Pref}(X)} \bigwedge_{\fsm{p} \in \mealyStates} x_{\lambda, p} \rightarrow o_{\fsm{p}, l, \fsm{r}}
\end{align}

In order to ensure that the prospective reward machine $\machine_{\mathcal{I}}$ and the LLM-generated DFA $D$ are synchronized we add auxiliary variables $y^{D}_{v,v^{\prime}}$ for $v \in V$ and $v^{\prime} \in V_{D}$. IF there is a label sequence $\lambda$ where $\machine_{\mathcal{I}}: v_{I} \run{\lambda} q$ and $D: q_{I, D} \run{\lambda} q^{\prime}$ is set to true.
We denote the conjunction of constraints \eqref{for:consistent-1}-\eqref{for:consistent-3} by $\Phi^\mathit{X}_{n}$ which we later on use to show the consistency of the learned reward machine with $X$.
To obtain this behavior, we add the following constraints:

\begin{align}
	\label{for:compatible-1}
	y_{\fsm{q_I}, \fsm{q'_{\mealyCommonInput, \Advice}}}^\Advice \\
	\label{for:compatible-2}
	\bigwedge_{\fsm{p, q} \in \mealyStates} \bigwedge_{\mealyCommonInput \in \mealyInputAlphabet} \bigwedge_{\delta_\Advice(\mealyStateStyle{p'}, \mealyCommonInput) = \fsm{q'}} (y_{\mealyStateStyle{p}, \mealyStateStyle{p'}}^\Advice \land d_{\mealyStateStyle{p}, \mealyCommonInput, \mealyStateStyle{q}}) \rightarrow y_{\mealyStateStyle{q}, \fsm{q'}}^\Advice \\
	\label{for:compatible-3}
	\bigwedge_{\mealyStateStyle{p} \in \mealyStates} \bigwedge_{\substack{\delta_\Advice(\mealyStateStyle{p'}, \mealyCommonInput) = \fsm{q'} \\ \fsm{q'} \notin \fsm{F}_\Advice}} y_{\mealyStateStyle{p}, \mealyStateStyle{p'}}^\Advice \rightarrow \lnot \bigvee_{\substack{r \in \setOfSeenRewards_X \\ r > 0}} o_{\mealyStateStyle{p}, \mealyCommonInput, r}
\end{align}

We denote the conjunction of constraints \eqref{for:compatible-1}-\eqref{for:compatible-3} by $\Phi^\mathit{X}_{n}$ which we later use to show the consistency of the learned reward machine with $X$. 
Theorem \ref{thm:RM-inference-correct} demonstrates the consistency of the LLM-generated DFA with the Algorithm \ref{alg:LARL-RM_alg}. W
e denote the conjunction
of Formulas \eqref{for:consistent-1}, \eqref{for:consistent-2}, and \eqref{for:consistent-3} by $\Phi^{D}_{n}$.

\begin{theorem} \label{thm:RM-inference-correct}
Consider $X \subset (2^\mathcal P)^+ \times \mathbb R^+$ be sample and $\mathscr{D}$ be a finite set of LLM-generated DFAs that are compatible with $X$.
Then, the following holds:
\begin{enumerate}
	\item \label{thm:RM-inference-correct:1}
	If we have $\mathcal I \models \Phi_n^{X, D}$, then the reward machine $\machine_\mathcal I$ is consistent with $X$ and compatible with the LLM-generated DFA $\Advice \in \mathscr{D}$.
	\item \label{thm:RM-inference-correct:2}
	If there exists a reward machine with $n$ states that is consistent with $X$ as well as compatible with each $\Advice \in \mathscr{D}$, then $\Phi_n^{X, D}$ is satisfiable.
\end{enumerate}
\end{theorem}

Proof for Theorem \ref{thm:RM-inference-correct} will follow in Section \ref{sec:learn_rm} after discussion of prerequisites.


\section{Learning Reward Machines} \label{sec:learn_rm}
In this section, we prove the correctness of our SAT-based algorithm for learning reward machines.
To this end, we show that the propositional formula $\Phi_n^X$ and $\Phi_n^\Advice$ have the desired meaning.
We begin with the formula $\Phi_n^X$, which is designed to enforce that the learned reward machine is consistent with the sample $X$.

\begin{lemma} \label{lem:RM-inference-consistent}
Let $\mathcal I \models \Phi_n^\mathit{RM} \land \Phi_n^X$ and $\machine_\mathcal I$ the reward machine defined above.
Then, $\machine_\mathcal I$ is consistent with $X$ (i.e., $\machine_\mathcal I(\lambda) = \rho$ for each $(\lambda, \rho) \in X$).
\end{lemma}

\begin{proof}
Let $\mathcal I \models \Phi_n^\mathit{RM} \land \Phi_n^X$ and $\machine_\mathcal I$ the reward machine constructed as above.
To prove Lemma~\ref{lem:RM-inference-consistent}, we show the following, more general statement by induction over the length of prefixes $(\lambda, \rho) \in \mathit{Pref}(X)$: if $\machine_\mathcal I \colon \fsm{q_I} \run{\lambda} \mealyStateStyle{q}$, then
\begin{enumerate}
	\item \label{lem:RM-inference-consistent:a}
	$\mathcal I(x_{\lambda, \mealyStateStyle{q}}) = 1$; and
	\item \label{lem:RM-inference-consistent:b}
	$\machine_\mathcal I(\lambda) = \rho$.
\end{enumerate}
Lemma~\ref{lem:RM-inference-consistent} then follows immediately from Part~\ref{lem:RM-inference-consistent:b} since $X \subseteq \mathit{Pref}(X)$.

\begin{description}
	\item[Base case:]
	Let $(\varepsilon, \varepsilon) \in \mathit{Pref}(X)$.
	By definition of runs, we know that $\machine_\mathcal I \colon \mealyInit \run{\varepsilon} \mealyInit$ is the only run of $\machine$ on the empty word.
	Similarly, Formula~\eqref{for:consistent-1} guarantees that the initial state $\mealyInit$ is the unique state $\mealyCommonState \in \mealyStates$ for which $\mathcal I(x_{\varepsilon, \mealyCommonState}) = 1$ holds.
	Both observations immediately prove Part~\ref{lem:RM-inference-consistent:a}.
	Moreover, $\machine(\varepsilon) = \varepsilon$ holds by definition the semantics of reward machines, which proves Part~\ref{lem:RM-inference-consistent:b}.
	\item[Induction step:]
	Let $(\lambda l, \rho r) \in \mathit{Pref}(X)$.
	Moreover, let $\machine_\mathcal I \colon \mealyInit \run{\lambda} \mealyStateStyle{p} \run{l} \mealyStateStyle{q}$ be the unique run of $\machine$ on $\lambda$.
	By applying the induction hypothesis, we then obtain that both $\mathcal I(x_{\lambda, p}) = 1$ and $\machine(\lambda) = \rho$ hold.
	
	To prove Part~\ref{lem:RM-inference-consistent:a}, note that $\machine_\mathcal I$ contains the transition $\mealyTransition(\mealyStateStyle{p}, l) = \mealyStateStyle{q}$ since this transition was used in the last step of the run on $\lambda$.
	By construction of $\machine_\mathcal I$, this can only be the case if $\mathcal I(d_{\mealyStateStyle{p}, l, \mealyStateStyle{q}}) = 1$.
	Then, however, Formula~\eqref{for:consistent-2} implies that $\mathcal I(x_{\lambda l, q}) = 1$ because $\mathcal I(x_{\lambda, p}) = 1$ (which holds by induction hypothesis).
	This proves Part~\ref{lem:RM-inference-consistent:a}.
	
	To prove Part~\ref{lem:RM-inference-consistent:b}, we exploit Formula~\eqref{for:consistent-3}.
	More precisely, Formula~\eqref{for:consistent-3} guarantees that if $\mathcal I(x_{\lambda, p}) = 1$ and the next input is $\mealyCommonInput$, then $\mathcal I(o_{p, \mealyCommonInput, r}) = 1$.
	By construction of $\machine_\mathcal I$, this means that $\sigma(\mealyStateStyle{q}, \mealyCommonInput) = r$.
	Hence, $\machine_\mathcal I$ outputs $r$ in the last step of the run on $\lambda l$.
	Since $\machine_\mathcal I(\lambda) = \rho$ (which holds by induction hypothesis), we obtain $\machine_\mathcal I(\lambda l) = \rho r$.
	This proves Part~\ref{lem:RM-inference-consistent:b}.
\end{description}
Thus, $\machine_\mathcal I$ is consistent with $X$.
\end{proof}

Next, we show that the formula $\Phi_n^\Advice$ ensures that the learned reward machine is compatible with the LLM-generated DFA $\Advice$.

\begin{lemma} \label{lem:RM-inference-compatible}
Let $\mathcal I \models \Phi_n^\mathit{RM} \land \Phi_n^\Advice$ and $\machine_\mathcal I$ the reward machine defined above.
Then, $\machine_\mathcal I$ is compatible with $\Advice$ (i.e., $\machine_\mathcal I(\labelSequence{k}) = \rewardSequence{k}$ and $r_k > 0$ implies $\labelSequence{k} \in L(\Advice)$ for every nonempty label sequence $\labelSequence{k}$).
\end{lemma}

\begin{proof}
Let $\mathcal I \models \Phi_n^\mathit{RM} \land \Phi_n^\Advice$ and $\machine_\mathcal I$ the reward machine defined above.
We first show that $\machine_\mathcal I \colon \mealyInit \run{\lambda} \mealyStateStyle{q}$ and $\Advice \colon \fsm{q_{I, \Advice}} \run{\lambda} \fsm{q'}$ imply $\mathcal I(y_{\mealyStateStyle{q}, \fsm{q'}}^\Advice) = 1$ for all label sequences $\lambda \in \mealyInputAlphabet^\ast$.
The proof of this claim proceeds by induction of the length of label sequences, similar to Part~\ref{lem:RM-inference-consistent:b} in the proof of Lemma~\ref{lem:RM-inference-consistent}.

\begin{description}
	\item[Base case:]
	Let $\lambda = \varepsilon$.
	By definition of runs, the only runs on the empty label sequence are $\machine_\mathcal I \colon \mealyInit \run{\varepsilon} \mealyInit$ and $\Advice \colon \fsm{q_{I, \Advice}} \run{\varepsilon} \fsm{q_{I, \Advice}}$.
	Moreover, Formula~\eqref{for:compatible-1} ensures that $\mathcal I(y_{\fsm{q_I}, \fsm{q_{I, \Advice}}}^\Advice) = 1$, which proves the claim.
	\item[Induction step:]
	Let $\lambda = \lambda^{\prime \prime} l$.
	Moreover, let $\machine_\mathcal I \colon \fsm{q_I} \run{\lambda^{\prime \prime}} \mealyStateStyle{p} \run{l} \mealyStateStyle{q}$ and $\Advice \colon \fsm{q_{I, \Advice}} \run{\lambda^{\prime \prime}} \fsm{p'} \run{l} \fsm{q'}$ be the runs of $\machine_\mathcal I$ and $\Advice$ on $\lambda = \lambda^{\prime \prime} l$, respectively.
	By induction hypothesis, we then know that $\mathcal I(y_{\mealyStateStyle{p}, \fsm{p'}}^\Advice) = 1$.
	Moreover, $\machine_\mathcal I$ contains the transition $\delta(\mealyStateStyle{p}, l) = \mealyStateStyle{q}$ because this transition was used in the last step of the run of $\machine_\mathcal I$ on $\lambda$.
	By construction of $\machine_\mathcal I$, this can only be the case if $\mathcal I(d_{\mealyStateStyle{p}, l, \mealyStateStyle{q}}) = 1$.
	In this situation, Formula~\ref{for:compatible-2} ensures $\mathcal I(y_{\mealyStateStyle{q}, \fsm{q'}}^\Advice) = 1$ (since also $\delta_\Advice(\fsm{p'}, l) = \fsm{q'}$), which proves the claim. 
\end{description}

Let now $\lambda =\labelSequence{k}$ be a nonempty label sequence (i.e., $k \geq 1$).
Moreover, let $\machine_\mathcal I \colon \fsm{q_I} \run{l_1 \ldots l_{k-1}} \mealyStateStyle{p} \run{l_k} \mealyStateStyle{q}$ be the run of $\machine_\mathcal I$ on $\lambda$ and $\Advice \colon \fsm{q_{I, \Advice}} \run{l_1 \ldots l_{k-1}} \fsm{p'} \run{l_k} \fsm{q'}$ the run of $\Advice$ on $\lambda$.
Our induction shows that $\mathcal I(y_{\mealyStateStyle{p}, \fsm{p'}}) = 1$ holds in this case.

Towards a contradiction, assume that $\machine_\mathcal I(\labelSequence{k}) = \rewardSequence{k}$, $r_k > 0$, and $\labelSequence{k} \notin L(\Advice)$.
In particular, this means $\fsm{q'} \notin \fsm{F}_\Advice$.
Since $\delta_\Advice(\fsm{p'}, l_k) = \fsm{q'}$ (which was used in the last step in the run of $\Advice$ on $\labelSequence{k}$) and $\mathcal I(y_{\mealyStateStyle{p}, \fsm{p'}}^\Advice) = 1$ (due to the induction above), Formula~\eqref{for:compatible-3} ensures that $\mathcal I(o_{\mealyStateStyle{p}, l, r}) = 0$ for all $r \in R_X$ with $r > 0$.
However, Formula~\eqref{for:well-defined-2} ensures that there is exactly one $r \in R_X$ with $\mathcal I(o_{\mealyStateStyle{p}, l, r}) = 1$.
Thus, there has to exist an $r \in R_X$ such that $r \leq 0$ and $\mathcal I(o_{\mealyStateStyle{p}, l, r}) = 1$.
By construction of $\machine_\mathcal I$, this means that the last output $r_k$ of $\machine_\mathcal I$ on reading $l_k$ must have been $r_k \leq 0$.
However, our assumption was $r_k > 0$, which is a contradiction.
Thus, $\machine_\mathcal I$ is compatible with the LLM-generated DFA $\Advice$.
\end{proof}

We can now prove Theorem~\ref{thm:RM-inference-correct} (i.e., the correctness of our SAT-based learning algorithm for reward machines).

\begin{proof}[Proof of Theorem~\ref{thm:RM-inference-correct}]
The proof of Part~\ref{thm:RM-inference-correct:1} follows immediately from Lemma~\ref{lem:RM-inference-consistent} and Lemma~\ref{lem:RM-inference-compatible}.

To prove Part~\ref{thm:RM-inference-correct:2}, let $\machine = (\mealyStates, \mealyInit, \mealyInputAlphabet, \mealyOutputAlphabet, \mealyTransition, \mealyOutput)$ be a reward machine with $n$ states that is consistent with $X$ and compatible with each $\Advice \in \mathscr{D}$.
From this reward machine, we can derive a valuation $\mathcal I$ for the variables $d_{\mealyStateStyle{p}, \mealyCommonInput, \mealyStateStyle{q}}$ and $o_{\mealyStateStyle{p}, \mealyCommonInput, r}$ in a straightforward way (e.g., setting $\mathcal I(d_{\mealyStateStyle{p}, \mealyCommonInput, \mealyStateStyle{q}}) = 1$ if and only is $\mealyTransition(\mealyStateStyle{p}, \mealyCommonInput) = \mealyStateStyle{q}$).
Moreover, we obtain a valuation for the variables $x_{\lambda, \mealyStateStyle{p}}$ from the runs of (prefixes) of traces in the sample $X$, and valuations for the variables $y_{\mealyStateStyle{p}, \fsm{p'}}^\Advice$ from the synchronized runs of $\machine$ and $\Advice$ for each $\Advice \in \mathscr{D}$.
Then, $\mathcal I$ indeed satisfies $\Phi_n^{X, D}$.
\end{proof}

\section{Convergence to an Optimal Policy}
In this section, we prove that \algname{} almost surely converges to an optimal policy in the limit.
We begin by defining attainable trajectories---trajectories that can possibly appear in the exploration of an agent.

\begin{definition}
Let 

$\mdp = (\mdpStates, \mdpInit, \mdpActions, \mdpProb, \mdpRewardFunction, \mdpDiscount, \rmLabels, \rmLabelingFunction)$ be a labeled MDP and $m \in \mathbb N$ a natural number.
A trajectory $\zeta = \trajectory{k} \in (S \times A)^\ast \times S$ is said to be \emph{$m$-attainable} if $k \leq m$ and $\mdpProb(s_{i-1}, a_i, s_i) > 0$ for each $i \in \{ 1, \ldots, k \}$.
Moreover, a trajectory $\zeta$ is called \emph{attainable} if there exists an $m \in \mathbb N$ such that $\zeta $ is $m$-attainable.
Analogously, we call a label sequence $\lambda = l_1 \ldots l_k$ \emph{($m$-)attainable} if there exists an ($m$-)attainable trajectory $s_0 a_1 s_1 \ldots s_{k-1} a_k s_k$ such that $l_i = L(s_{i-1}, a_i, s_i)$ for each $i \in \{ 1, \ldots, k \}$
\end{definition}

An induction shows that \algname\ almost surely explores every attainable trajectory in the limit 
(i.e., with probability 1 when the number of episodes goes to infinity.
We refer the reader to the extended version of \cite{DBLP:conf/aips/0005GAMNT020}\footnote{The extended version of \cite{DBLP:conf/aips/0005GAMNT020} can be found on arXiv at \url{https://arxiv.org/abs/1909.05912}} for a detailed proof.

\begin{lemma} \label{lem:full-exploration-trajectories}
Given $m \in \mathbb N$, \algname\ with $\eplen \geq m$ almost surely explores every $m$-attainable trajectory in the limit.
\end{lemma}

As an immediate consequence of Lemma~\ref{lem:full-exploration-trajectories}, we obtain that \algname\ almost sure explores every \hbox{($m$-)}attainable label sequence in the limit as well.

\begin{corollary} \label{cor:full-exploration-label-sequences}
Given $m \in \mathbb N$, \algname\ with $\eplen \geq m$ almost surely explores every $m$-attainable label sequence in the limit.
\end{corollary}

In this paper, we assume that 
all possible label sequences are attainable.
We relax this restriction later as in \cite{neider2021advice} .

To prove Lemma~\ref{lem:learn-correct-RM}, we show a series of intermediate results.
We begin by reminding the reader of a well-known fact from automata theory, which can be found in most textbooks~\cite{DBLP:books/daglib/0025557}.

\begin{theorem} \label{thm:symmetric-difference}
Let $\Automaton_1$ and $\Automaton_2$ be two DFAs with $L(\Automaton_1) \neq L(\Automaton_2)$.
Then, there exists a word $w$ of length at most $|\Automaton_1| + |\Automaton_2| - 1$ such that $w \in L(\Automaton_1)$ if and only if $w \notin L(\Automaton_2)$.
\end{theorem}

We now observe that every reward machine $\machine$ can be translated into a DFA $\mathfrak A_\machine$ that is ``equivalent'' to the reward machine \cite{DBLP:conf/aips/0005GAMNT020}.
This DFA operates over the combined alphabet $2^\mathcal P \times \mealyOutputAlphabet$ and accepts a sequence $(\ell_1, r_1) \ldots (\ell_k,r_k)$ if and only if $\machine$ outputs the reward sequence $\rewardSequence{k}$ on reading the label sequence $\labelSequence{k}$.

\begin{lemma}[\cite{DBLP:conf/aips/0005GAMNT020}] \label{lem:RM-to-DFA}
Given a reward machine $\machine = (\mealyStates, \mealyInit, \mealyInputAlphabet, \mealyOutputAlphabet, \mealyTransition, \mealyOutput)$, one can construct a DFA $\Automaton_\machine$ with $|\machine| + 1$ states such that $L(\Automaton_\machine) = \{ (l_1, r_1) \ldots (l_k, r_k) \in (2^\mathcal P \times \mathbb R) \mid \machine(l_1 \ldots l_k) = r_1 \ldots r_k \}$.
\end{lemma}

\begin{proof}
Let $\machine= (\fsm{Q}_\machine, \fsm{q_{I, A}}, 2^\mathcal P, \mathbb R, \delta_\machine, \sigma_\machine)$ be a reward machine. 
Then, we define a DFA $\mathfrak A_\machine = (\fsm{Q}, \fsm{q_I}, 2^\mathcal P \times \mathbb R, \delta, F)$ by
\begin{itemize}
	\item $\fsm{Q} = \fsm{Q_A} \cup \{ \bot \}$, where $\bot \notin \fsm{Q_A}$;
	\item $\fsm{q_I} = \fsm{q_{I, A}}$;
	\item $\delta \bigl( \mealyStateStyle{p}, (\ell,r) \bigr) = \begin{cases} \mealyStateStyle{q} & \text{if $\delta_\machine(\mealyStateStyle{p}, \ell) = \mealyStateStyle{q}$ and $\sigma_\machine (\mealyStateStyle{p}, \ell) = r$;} \\ \bot & \text{otherwise}\end{cases}$
	\item $\fsm{F}_\machine = \fsm{Q_A}$.
\end{itemize}

In this definition, $\bot$ is a new sink state to which $\mathfrak A_\machine$ moves if its input does not correspond to a valid input-output pair produced by $\machine$.
A straightforward induction over the length of inputs to $\mathfrak A_\machine$ shows that it indeed accepts the desired language.
In total, $\mathfrak A_\machine$ has $|\machine|+ 1$ states.
\end{proof}

Next, we show that if two reward machines are semantically different, then we can bound the length of a trace witnessing this fact.

\begin{lemma} \label{lem:short-witness-to-non-consistency}
Let $\machine_1$ and $\machine_2$ be two reward machines.
If $\machine_1 \neq \machine_2$, then there exists a label sequence $\lambda$ of length at most $|\machine_1| + |\machine_2| + 1$ such that $\machine_1(\lambda) \neq \machine_2(\lambda)$.
\end{lemma}

\begin{proof}[Proof of Lemma~\ref{lem:short-witness-to-non-consistency}]
Consider the DFAs $\Automaton_{\machine_1}$ and $\Automaton_{\machine_2}$ obtained from Lemma~\ref{lem:RM-to-DFA}.
These DFAs have $|\machine_1| + 1$ and $|\machine_2| + 1$ states, respectively.
If $\machine_1 \neq \machine_2$, then $L(\Automaton_{\machine_1}) \neq L(\Automaton_{\machine_2})$.
Thus, applying Theorem~\ref{thm:symmetric-difference} yields a sequence $w = (l_1, r_1) \ldots (l_k, r_k)$ with
\[ k \leq |\machine_1| + 1 + |\machine_2| + 1 - 1 = |\machine_1| + |\machine_2| + 1 \]
such that $w \in L(\Automaton_{\machine_1})$ if and only if $w \notin L(\Automaton_{\machine_2})$.
By Lemma~\ref{lem:RM-to-DFA}, the label sequence $\lambda = l_1 \ldots l_k$ has the desired property.
\end{proof}

Similar to Lemma~\ref{lem:short-witness-to-non-consistency}, we can show that if an LLM-generated DFA is not compatible with a reward machine, we can also bound the length of a label sequence witnessing this fact.

\begin{lemma} \label{lem:short-witness-to-non-compatibility}
Let $\machine$ be a reward machine and $\Advice$ an LLM-generated DFA.
If $\Advice$ is not compatible with $\machine$, then there exists a label sequence $l_1 \ldots l_k$ with $k \leq 2(|\machine| + 1) \cdot |\Advice|$ such that $\machine(l_1 \ldots l_k) = r_1 \ldots r_k$, $r_k > 0$, and $l_1 \ldots l_k \notin L(\Advice)$.
\end{lemma}

\begin{proof}
Let $\machine$ be a reward machine and $R \subset \mathbb R$ the \emph{finite} set of rewards that $\machine$ can output.
Moreover, let $\Advice = (V_\Advice, v_{I, \Advice}, 2^\mathcal P, \delta_\Advice, F_\Advice)$ be an LLM-generated DFA.
Our proof proceeds by constructing four DFAs and a subsequent analysis of the final DFA to derive the desired bound.

First, we construct the DFA $\Automaton_{\machine} = (V_\machine, v_{I, \machine}, 2^\mathcal P \times R, \delta_\machine, F_\machine)$ according to Lemma~\ref{lem:RM-to-DFA}.
Recall that $(l_1, r_1) \ldots (l_k, r_k) \in L(\Automaton_\machine)$ if and only if $\machine(l_1 \ldots l_k) = r_1 \dots r_k$.
Moreover, $\Automaton_\machine$ has $|\machine| + 1$ states.

Second, we modify the DFA $\Automaton_\machine$ such that it only accepts sequences $(l_1, r_1) \ldots (l_k, r_k)$ with $r_k > 0$.
To this end, we augment the state space with an additional bit $b \in \{ 0, 1 \}$, which tracks whether the most recent reward was $0$ or greater than $0$.
More formally, we define a DFA $\Automaton'_\machine = (V'_\machine, v'_{I, \machine}, 2^\mathcal P \times R, \delta'_\machine, F'_\machine)$ by
\begin{itemize}
	\item $V'_\machine = V_\machine \times \{ 0, 1 \}$;
	\item $v'_{I, \machine} = (v_{I, \machine}, 0)$;
	\item $\delta_\machine' \bigl( (v, b), (l, r) \bigr) = \bigl( \delta_\machine(v, (l, r)), b \bigr)$ where $b = 1$ if and only if $r > 0$; and
	\item $F'_\machine = F_\machine \times \{ 1 \}$.
\end{itemize}
It is not hard to verify that $\Automaton'_\machine$ indeed has the desired property.
Moreover, by Lemma~\ref{lem:RM-to-DFA}, $\Automaton'_\machine$ has $2(|\machine| + 1)$ states.

Third, we apply ``cylindrification'' to the LLM-generated DFA $\Advice$, which works over the alphabet $2^\mathcal P$, to match the input alphabet $2^\mathcal P \times R$ of $\Automaton'_\machine$.
Our goal is to construct a new DFA $\Advice' = (V'_\Advice, v'_{I, \Advice}, 2^\mathcal P \times R(X), \delta'_\Advice, F'_\Advice)$ that accepts a sequence $(l_1, r_1) \ldots (l_k, r_k)$ if and only if $l_1 \ldots l_k \in L(\Advice)$ for every reward sequence $r_1 \ldots r_k$ (cylindrification can be thought of as the inverse of the classical projection operation).
We achieve this by replacing each transition $\delta_\Advice(v, l) = v'$ in $\Advice$ with $|R(X)|$ many transitions of the form $\delta'_\Advice\bigl( v, (l, r) \bigr)$ where $r \in R$.
Formally, we define the DFA $\Advice'$ by
\begin{itemize}
	\item $V'_\Advice = V_\Advice$;
	\item $v'_{I, \Advice} = v_{I, \Advice}$;
	\item $\delta'_\Advice \bigl( v, (l, r)\bigr) = \delta_\Advice(v, l)$ for each $r \in R$; and
	\item $F'_\Advice = F_\Advice$.
\end{itemize}
It is not hard to verify that $\Advice'$ has indeed the desired property and its size is $|\Advice|$.

Fourth, we construct the simple product DFA of $\Automaton'_\machine$ and $\Advice'$.
This DFA is given by $\Automaton = (V, v_I, 2^\mathcal P \times R(X), \delta, F)$ where
\begin{itemize}
	\item $V = V'_\machine \times V'_\Advice$;
	\item $v_I = (v'_{I, \machine}, v'_{I, \Advice})$;
	\item $\delta \bigl( (v_1, v_2), (l, r) \bigr) = \bigl( \delta'_{\machine}(v_1, (l, r)), \delta'_{\Advice}(v_2, (l, r)) \bigr)$; and
	\item $F = F'_\machine \times (Q'_\Advice \setminus F'_\Advice)$.
\end{itemize}
By construction of $\Automaton'_\machine$ and $\Advice'$, is is not hard to see that $\Automaton$ accepts a sequence $(l_1, r_1) \ldots (l_k, r_k)$ if and only if $\machine(l_1 \ldots l_k) = r_1 \ldots r_k$ with $r_k > 0$ and $l_1 \ldots l_k \notin L(\Advice)$---in other words, $L(\Automaton)$ contains all sequences that witness that $\Advice$ is not compatible with $\machine$.
Moreover, $\Automaton$ has $2(|\machine| + 1) \cdot |\Advice|$ states.

It is left to show that if $\Advice$ is not compatible with $\machine$, then we can find a witness with the desired length.
To this end, it is sufficient to show that if $L(\Automaton) \neq \emptyset$, then there exists a sequence $(l_1, r_1) \ldots (l_k, r_k) \in L(\Automaton)$ with $k \leq 2(|\machine| + 1) \cdot |\Advice|$.
This fact can be established using a simple pumping argument.
To this end, assume that $(l_1, r_1) \ldots (l_k, r_k) \in L(\Automaton)$ with $k > 2(|\machine| + 1) \cdot |\Advice|$.
Then, there exists a state $v \in V$ such that the unique accepting run of $\Automaton$ on $(l_1, r_1) \ldots (l_k, r_k)$ visits $v$ twice, say at the positions $i, j \in \{ 0, \ldots k \}$ with $i < j$.
In this situation, however, the DFA $\Automaton$ also accepts the sequence $(l_1, r_1) \ldots (l_i, r_i) (l_{j+1}, r_{j+1}) \ldots (l_k, r_k)$, where we have removed the ``loop'' between the repeating visits of $v$.
Since this new sequence is shorter than the original sequence, we can repeatedly apply this argument until we arrive at a sequence $(l'_1, r'_1) \ldots (l'_\ell, r'_\ell) \in L(\Automaton)$ with $\ell \leq 2(|\machine| + 1) \cdot |\Advice|$.
By construction of $\Automaton$, this means that
$\machine(l'_1 \ldots l'_\ell) = r'_1 \ldots r'_\ell$, $r_\ell > 0$, and $l'_1 \ldots l'_\ell \notin L(\Advice)$, which proves the claim.
\end{proof}

With these intermediate results at hand, we are now ready to prove Lemma~\ref{lem:learn-correct-RM}.

\begin{proof}[Proof of Lemma~\ref{lem:learn-correct-RM}]
Let $(X_0, D_0), (X_1, D_1), \ldots$ be the sequence of samples and sets of LLM-generated DFAs that arise in the run of \algname{} whenever a new counterexample is added to $X$ (in Lines~\ref{line:counterExampleRM} and \ref{line:RMUpdateTrace} of Algorithm~\ref{alg:LARL-RM_alg}) or an LLM-generated DFA is removed from the set $D$ (Lines~\ref{line:checkUpdatedLists} and \ref{line:DFAcounter} of Algorithm~\ref{alg:LARL-RM_alg}).\
Moreover, let $\machine_0, \machine_1, \ldots$ be the corresponding sequence of reward machines that are computed from $(X_i, D_i)$.
Note that constructing a new reward machine is always possible because \algname{} makes sure that all LLM-generated DFAs in the current set $\mathscr{D}$ are compatible with the traces in the sample $X$.

We first observe three properties of these sequences:
\begin{enumerate}
	\item \label{prop:alg-correct:1}
	The true reward machine $\machine$ (i.e., the one that encodes the reward function $R$) is consistent with every sample $X_i$ that is generated during the run of \algname.
	This is due to the fact that each counterexample is obtained from an actual exploration of the MDP and, hence, corresponds to the ``ground truth''.
	\item \label{prop:alg-correct:2}
	The sequence $X_0, X_1, \ldots$ grows monotonically (i.e., $X_0 \subseteq X_1 \subseteq \cdots$) because \algname\ always adds counterexamples to $X$ and never removes them (Lines~\ref{line:counterExampleRM} and \ref{line:RMUpdateTrace}).
	In fact, whenever a counterexample $(\lambda^{\prime}, \rho)$ is added to $X_i$ to form $X_{i+1}$, then $(\lambda^{\prime}, \rho) \notin X_i$ (i.e., $X_i \subsetneq X_{i+1}$).
	To see why this is the case, remember that \algname\ always constructs hypotheses that are consistent with the current sample (and the current set of LLM-generated DFAs).
	Thus, the current reward machine $\machine_i$ is consistent with $X_i$, but the counterexample $(\lambda^{\prime}, \rho)$ was added because $\machine_i(\lambda^{\prime}) \neq \rho$.
	Thus, $(\lambda^{\prime}, \rho)$ cannot have been an element of $X_i$.
	\item \label{prop:alg-correct:3}
	The sequence $D_0, D_1, \ldots$ decreases monotonically (i.e., $D_0 \supseteq D_1 \supseteq \cdots$) because \algname\ always removes LLM-generated DFAs from $D$ and never adds any (Lines~\ref{line:checkUpdatedLists} and \ref{line:DFAcounter}).
	Thus, there exists a position $i^\star \in \mathbb N$ at which this sequence becomes stationary, implying that $D_i = D_{i+1}$ for $i \geq i^\star$.
\end{enumerate}

Similar to Property~\ref{prop:alg-correct:1}, we now show that each LLM-generated DFA in the set $D_i$, $i \geq i^\star$, is compatible with the true reward machine $\machine$.
Towards a contradiction, let $\Advice \in \mathscr{D}_i$ be an LLM-generated DFA and assume that $\Advice$ is not compatible with $\machine$.
Then, Lemma~\ref{lem:short-witness-to-non-compatibility} guarantees the existence of a label sequence $l_1 \ldots l_k$ with
\begin{align*}
	k & \leq 2(|\machine| + 1) \cdot |\Advice| \\
	& \leq 2(|\machine| + 1) \cdot n_\text{max}
\end{align*}
such that $\machine(l_1, \ldots l_k) = r_1 \ldots r_k$,$r_k > 0$, and $l_1 \ldots l_k \notin L(\Advice)$.
Since we assume all label sequences to be attainable and have chosen $\eplen \geq 2(|\machine| + 1) \cdot n_\text{max}$, Corollary~\ref{cor:full-exploration-label-sequences} guarantees that \algname\ almost surely explores this label sequence in the limit.
Once this happens, \algname\ removes $\Advice$ from the set $\mathscr{D}$ (Lines~\ref{line:checkUpdatedLists} and \ref{line:DFAcounter}), which is diction to the fact that the sequence $D_{i^\star}, D_{i^\star+1}, \ldots$ is stationary (Property~\ref{prop:alg-correct:3}).
Hence, we obtain the following:
\begin{enumerate}[resume]
	\item \label{prop:alg-correct:4}
	Every LLM-generated DFA in $D_i$, $i \geq i^\star$, is compatible with the true reward machine $\machine$.
\end{enumerate}

Next, we establish the three additional properties about the sub-sequence $\machine_{i^\star}, \machine_{i^\star + 1}$ of hypotheses starting at position $i^\star$:
\begin{enumerate}[resume]
	\item \label{prop:alg-correct:5}
	The size of the true reward machine $\machine$ is an upper bound for the size of $\machine_{i^\star}$ (i.e., $|\machine_{i^\star}| \leq |\machine|$).
	This is due to the fact that $\machine$ is consistent with every sample $X_i$ (Property~\ref{prop:alg-correct:1}), every LLM-generated DFA in $\mathscr{D}_i$, $i \geq i^\star$, is compatible with $\machine$ (Property~\ref{prop:alg-correct:4}), and \algname{} always computes minimal consistent reward machines.
	\item \label{prop:alg-correct:6}
	We have $|\machine_i| \leq |\machine_{i+1}|$ for all $i \geq i^\star$.
	Towards a contradiction, assume that $|\machine_i| > |\machine_{i+1}|$.
	Since \algname\ always computes consistent reward machines and $X_i \subsetneq X_{i+1}$ if $i \geq i^\star$ (see Property~\ref{prop:alg-correct:2}), we know that $\machine_{i + 1}$ is not only consistent with $X_{i+1}$ but also with $X_i$ (by definition of consistency).
	Moreover, \algname\ computes minimal consistent reward machines.
	Hence, since $\machine_{i + 1}$ is consistent with $X_i$ and $|\machine_{i + 1}| < |\machine_i|$, the reward machine $\machine_i$ is not minimal, which is a contradiction.
	\item \label{prop:alg-correct:7}
	We have $\machine_i \neq \machine_j$ for $i \geq i^\star$ and $j \in \{ i^\star, \ldots, i \}$---in other words, the reward machines generated during the run of \algname\ after the $i^\star$-th recomputation are semantically distinct.
	This is a consequence of the facts that $(\lambda^{\prime}_j, \rho_j)$ was a counterexample to $\machine_j$ (i.e., $\machine_j(\lambda^{\prime}_j) \neq \rho_j)$ and that \algname always constructs consistent reward machines (which implies $\machine_i(\lambda^{\prime}_i) = \rho_i$).
\end{enumerate}

Properties~\ref{prop:alg-correct:5} and \ref{prop:alg-correct:6} now provide $|\machine|$ as an upper bound on the size of any reward machine that \algname\ constructs after the $i^\star$-th recomputation.
Since there are only finitely many reward machines of size at most $|\machine|$, Property~\ref{prop:alg-correct:7} implies that there exists a $j^\star \geq i^\star$ after which no new reward machine is learned.
Hence, it is left to show that $\machine_{j^\star} = \machine$ (i.e., $\machine_{j^\star}(\lambda) = \machine(\lambda)$ for all label sequences $\lambda$).

Towards a contradiction, let us assume that $\machine_{j^\star} \neq \machine$.
Then, Lemma~\ref{lem:short-witness-to-non-consistency} guarantees the existence of a label sequence $\lambda = l_1 \ldots l_k$ with
\begin{align*}
	k & \leq |\machine_{j^\star}| + |\machine| + 1 \\
	& \leq 2|\machine| + 1 \\
	& \leq 2( |\machine| + 1) \cdot n_\text{max}
\end{align*}
such that $\machine_{j^\star}(\lambda) \neq \machine(\lambda)$.

Since we assume all label sequences to be attainable and have chosen $\eplen \geq 2(|\machine| + 1) \cdot n_\text{max}$, Corollary~\ref{cor:full-exploration-label-sequences} guarantees that \algname\ almost surely explores this label sequence in the limit.
Thus, the trace $(\lambda, \rho)$, where $\rho = \machine(\lambda)$, is almost surely returned as a new counterexample, resulting in a new sample $X_{j^\star + 1}$.
This, in turn, causes the construction of a new reward machine, which contradicts the assumption that no further reward machine is generated.
Thus, the reward machine $\machine_{j^\star}$ is equivalent to the true reward machine $\machine$.
\end{proof}

Let us finally turn to the proof of Theorem~\ref{th:optimal_policy} (i.e., that \algname\ almost surely converges to the optimal policy in the limit).
The key idea is to construct the product of the given MDP $\mathcal M$ and the true reward machine $\machine$.
In this new MDP $\mathcal M'$, the reward function is in fact Markovian since the reward machine $\machine$ encodes all necessary information in its states (and, hence, in the product $\mathcal M'$).
Thus, the fact that classical Q-learning almost surely converges to an optimal policy in the limit also guarantees the same for \algname.
We refer the reader to the extended version of \cite{DBLP:conf/aips/0005GAMNT020} for detailed proof.

\section{Unattainable Label Sequences}
\label{sec:unattainablesequences}
Let us now consider the case that not all labels can be explored by the agent (e.g., because certain label sequences are not possible in the MDP).
In this situation, two complications arise:
\begin{enumerate}
	\item
	It is no longer possible to uniquely learn the true reward machine $\machine$.
	In fact, any reward machine that agrees with $\machine$ on the attainable traces becomes a valid solution.
	\item
	The notion of compatibility needs to reflect the attainable traces.
	More precisely, the condition $r_k > 0$ implies $l_1 \ldots l_k \in L(\Advice)$ now has to hold only for attainable label sequences.	
\end{enumerate}
Both complications add a further layer of complexity, which makes the overall learning problem---both Q-learning and the learning of reward machines---harder.
In particular, we have to adapt the episode length and also the size of the SAT encoding grows.
In total, we obtain the following result.

\begin{theorem} \label{th:optimal_policy-unattainable}
Let $\mdp$ be a labeled MDP where all label sequences are attainable and $\machine$ the true reward machine encoding the rewards of $\mdp$.
Moreover, let $\mathscr{D} = \{ \Advice_1, \ldots, \Advice_\ell \}$ be a set of LLM-generated DFAs and $n_\text{max} = \max_{1 \leq i \leq \ell}{\{ |\Advice_i| \}}$.
Then, \algname\ with
\[ \eplen \geq \max{ \bigl\{ 2 |\mdp| \cdot ( |\machine| + 1) \cdot n_\text{max}, |\mdp|(|\machine| + 1)^2 \bigr\}} \]
almost surely converges to an optimal policy in the limit.
The size of the formula $\Phi_n^{X, D}$ grows linearly in $|\mdp|$.
\end{theorem}

In order to prove Theorem~\ref{th:optimal_policy-unattainable}, specifically the bound on the length of episodes, we first need to introduce the nondeterministic version of DFAs.
Formally, a \emph{nondeterministic finite automaton} is a tuple $\mathfrak A = (\fsm{Q}, \fsm{q_I}, \Sigma, \Delta, F)$ where $\fsm{Q}$, $\fsm{q_I}$, $\Sigma$, $\fsm{F}$ are as in DFAs and $\Delta \subset \fsm{Q} \times \Sigma \times \fsm{Q}$ is the transition relation.
Similar to DFAs, a run of an NFA $\mathfrak A$ on a word $u = a_1 \ldots a_k$ is a sequence $\mealyStateStyle{q}_0, \ldots \mealyStateStyle{q}_k$ such that $\fsm{q_o} = \fsm{q_I}$ and $(\mealyStateStyle{q}_{i-1}, a_i, \mealyStateStyle{q}_i) \in \Delta$ for each $i \in \{ 1, \ldots, k\}$.
In contrast to DFAs, however, NFAs permit multiple runs on the same input or even no run.
Accepting runs and the language of NFAs are defined analogously to DFAs.
Note that we use NFAs instead of DFAs because the former can be exponentially more succinct than the latter.

Similar to Lemma~\ref{lem:RM-to-DFA}, one can translate an MDP $\mdp$ into an NFA $\Automaton_\mdp$ with $|\mdp|$ states that accepts exactly the attainable traces of an MDP $\mathcal M$.

\begin{lemma} \label{lem:MDP-to-NFA}
Given a labeled MDP $\mathcal M$, one can construct an NFA $\Automaton_\mathcal M$ with at most $2|\mathcal M|$ states that accept exactly the admissible label sequences of $\mathcal M$.
\end{lemma}

This construction is similar to Remark~1 of the extended version of \cite{DBLP:conf/aips/0005GAMNT020} as proceeds as follows.

\begin{proof}
Given a labeled MDP $\mdp = (\mdpStates, \mdpInit, \mdpActions, \mdpProb, \mdpRewardFunction, \mdpDiscount, \rmLabels, \rmLabelingFunction)$, we construct an NFA $\Automaton_\mathcal M = (\fsm{Q}, \fsm{q_I}, 2^\mathcal P, \Delta, F)$ by
\begin{itemize}
	\item $\fsm{Q} = \mdpStates$;
	\item $\fsm{q_I} = \mdpInit$;
	\item $(s, \ell, s') \in \Delta$ if and only if there exists an action $a \in \mdpActions$ with $\rmLabelingFunction(s, a, s') = \ell$ and $\mdpProb(s, a, s) > 0$; and 
	\item $\fsm{F} = \mdpStates$.
\end{itemize}
A straightforward induction shows that $\lambda \in L(\Automaton_\mathcal M)$ holds if and only if $\lambda$ is an attainable label of $\mathcal M$.
\end{proof}

We are now ready to prove Theorem~\ref{th:optimal_policy-unattainable}.

\begin{proof}[Proof of Theorem~\ref{th:optimal_policy-unattainable}]
We first modify Lemma~\ref{lem:short-witness-to-non-compatibility} slightly.
More precisely, we add another fifth step, where we build the product of $\Automaton$ with the cylindrification of the DFA $\Automaton_\mdp$.
It is not hard to verify that this results in an NFA that accepts a trace $(l_1 \ldots l_k, r_1 \ldots r_k)$ if and only if $l_1 \ldots l_k$ is admissible and not in the language $L(\Advice)$ , $\machine(l_1  \ldots l_k) = r_1 \ldots r_k$, and $r_k > 0$.
This NFA has $2|\mdp|(|\machine| + 1) \cdot |\Advice|$ states.
A similar pumping argument as in the proof of Lemma~\ref{lem:short-witness-to-non-compatibility} can now be used to show the existence of a witness of length at most $2|\mdp|(|\machine| + 1) \cdot |\Advice|$.

Analogously, we can modify Lemma~\ref{lem:short-witness-to-non-consistency}.
We build the input-synchronized product of the two DFAs $\Automaton_{\machine_1}$ and $\Automaton_{\machine_2}$ and the DFA $\Automaton_\mdp$.
This results in an NFA with $(|\machine_1| + 1) \cdot (|\machine_2| + 1) \cdot |\mdp|$ states.
If $\machine_1 \neq \machine_2$ on an attainable label sequence, then we can find a sequence such that leads in this product to a state where $\Automaton_\mdp$ accepts (the label sequence is attainable), but exactly one of $\machine_1$ and $\machine_2$ accepts.
Using a pumping argument as above, we can bound the length of such an input to at most $(|\machine_1| + 1) \cdot (|\machine_2| + 1) \cdot |\mdp|$.

Moreover, we have to modify Formula $\Phi_n^{X, D}$ to account for the new situation.
Given $\Automaton_\mdp = (\fsm{Q_{A_\mdp}}, \fsm{q_{I, A_\mdp}}, 2^\mathcal P \times \mathbb R, \Delta_{\fsm{A_\mdp}}, F_{\fsm{A_\mdp}})$, we first introduce new auxiliary variables $z_{\mealyStateStyle{p}, \fsm{p'}, \fsm{p''}}^\Advice$ where $\fsm{p''} \in \fsm{Q_{\machine_\mdp}}$;
we use these variables instead of the variables $y_\mealyStateStyle{p}, \fsm{p'}^\Advice$.
Second, we replace Formulas~\eqref{for:compatible-1}, \eqref{for:compatible-2}, and \eqref{for:compatible-3} with the following three formulas:
\begin{align}
&	z_{\fsm{q_I}, \fsm{q'_{I, \Advice}, \fsm{q_{I, A_\mdp}}}}^\Advice \\
&	\bigwedge_{\fsm{p, q} \in \fsm{Q}} \bigwedge_{l \in 2^\mathcal P} \bigwedge_{\delta_\Advice(\fsm{p'}, l) = \fsm{q'}} \bigwedge_{(\fsm{p''}, l, \fsm{q''}) \in \Delta_\fsm{A_\mdp}} (z_{\mealyStateStyle{p}, \fsm{p'}, \fsm{p''}}^\Advice \land d_{\mealyStateStyle{p}, l, \mealyStateStyle{q}}) \rightarrow z_{\mealyStateStyle{q}, \fsm{q'}, \fsm{q''}}^\Advice \\
&	\bigwedge_{\mealyStateStyle{p} \in \fsm{Q}} \bigwedge_{\substack{\delta_\Advice(\fsm{p'}, l) = \fsm{q'} \\
 \fsm{q'} \notin \fsm{F}_\Advice}} \bigwedge_{\substack{(\fsm{p''}, l, \fsm{q''}) \in \Delta_\fsm{A_\mdp} \\
 \fsm{q''} \in F_{\fsm{A_\mdp}}}} z_{\mealyStateStyle{p}, \fsm{p'}, \fsm{p''}}^\Advice \rightarrow \lnot \bigvee_{\substack{r \in R_X 
 r > 0}}o_{\mealyStateStyle{p}, l, r}
\end{align}
Clearly, the size of $\Phi_n^{X, D}$ grows linearly in $|\mdp|$ as compared to the case that all label sequences are attainable.
Moreover, it is not hard to verify that $\Phi_n^{X, D}$ has indeed the desired meaning.
More precisely, we obtain a result analogous to Lemma~\ref{lem:learn-correct-RM}, but for the case that not all label sequences are attainable.

Once we have established that our learning algorithm learns consistent and compatible reward machines, the overall convergence of \algname\ to an optimal policy and the bound on the length of episodes can then be proven analogously to Theorem~\ref{th:optimal_policy}.
\end{proof}

\end{document}